\documentclass[sigconf]{acmart}

\AtBeginDocument{%
  }

\copyrightyear{2022}
\acmYear{2022}
\setcopyright{rightsretained}
\acmConference[MM '22]{Proceedings of the 30th ACM International Conference on Multimedia}{October 10--14, 2022}{Lisboa, Portugal}
\acmBooktitle{Proceedings of the 30th ACM International Conference on Multimedia (MM '22), Oct. 10--14, 2022, Lisboa, Portugal}
\acmISBN{978-1-4503-9203-7/22/10}
\acmDOI{10.1145/3503161.3548032}

\newcommand*{\Scale}[2][4]{\scalebox{#1}{$#2$}}%
\newcommand{\myparagraph}[1]{\textit{\textbf{#1}}}

\usepackage{amsfonts}

\usepackage{array}
\usepackage{textcomp}
\usepackage{stfloats}
\usepackage{url}
\usepackage{verbatim}

\usepackage{mathrsfs}
\usepackage{bm}

\usepackage{amsmath}

\usepackage{graphicx}
\usepackage{hyperref}
\usepackage{subfig}
\usepackage{float}
\usepackage{stfloats}
\usepackage{booktabs}
\usepackage{color,colortbl}
\usepackage{dsfont}
\usepackage{multirow}
\usepackage{xspace}
\usepackage{mathtools}
\usepackage{balance}

\usepackage{enumitem}
\usepackage{adjustbox}
\usepackage{dashbox}

\usepackage{algorithm,algorithmic}

\usepackage{lipsum}
\usepackage{overpic}

\usepackage{etoolbox}
\makeatletter
\patchcmd{\maketitle}{\@copyrightpermission}{
   \begin{minipage}{0.3\columnwidth}
     \href{https://creativecommons.org/licenses/by/4.0/}{\includegraphics[width=0.90\textwidth]{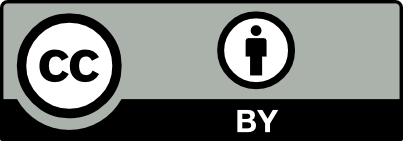}}
   \end{minipage}\hfill
   \begin{minipage}{0.7\columnwidth}
     \href{https://creativecommons.org/licenses/by/4.0/}{This work is licensed under a Creative Commons Attribution International 4.0 License.}
   \end{minipage}
 
   \vspace{5pt}
}{}{}

\makeatother

\makeatletter
\DeclareRobustCommand\onedot{\futurelet\@let@token\@onedot}
\def\@onedot{\ifx\@let@token.\else.\null\fi\xspace}

\def \eg{\emph{e.g}\onedot}

\def \ie{\emph{i.e}\onedot}

\def \cf{\emph{c.f}\onedot}

\def\vs{\emph{vs}\onedot}
\def \wrt{w.r.t\onedot}

\makeatother
\newcommand{\name}{{HyP$^2$ Loss}}

\newcommand{\uline}[1]{\underline{#1}}

\begin{document}

\title{{\name}: Beyond Hypersphere Metric Space for Multi-label Image Retrieval}

\author{Chengyin Xu}
\authornote{Equal contributions. Listing order is random.}
\authornote{Work done when the authors were interns at Tencent AI Lab.}
\affiliation{%
  \institution{SIGS, Tsinghua University}
  \country{}
}
\author{Zenghao Chai}
\authornotemark[1]
\authornotemark[2]
\affiliation{%
  \institution{SIGS, Tsinghua University}
  \country{}
}
\author{Zhengzhuo Xu}
\authornotemark[2]
\affiliation{%
  \institution{SIGS, Tsinghua University}
  \country{}
}
\author{Chun Yuan}
\authornote{Corresponding authors. yuanc@sz.tsinghua.edu.cn, fanyanbo0124@gmail.com}
\affiliation{%
  \institution{SIGS, Tsinghua University}
  \country{}
}
\affiliation{%
  \institution{Peng Cheng National Laboratory}
  \country{}
}

\author{Yanbo Fan}
\authornotemark[3]
\affiliation{%
  \institution{Tencent AI Lab}
  \country{}
}
\author{Jue Wang}
\affiliation{%
  \institution{Tencent AI Lab}
  \country{}
}

\setlength{\textfloatsep}{6pt}
\setlength{\floatsep}{8pt}
\setlength{\intextsep}{10pt}
\setlength{\abovecaptionskip}{0pt}
\setlength{\belowcaptionskip}{0pt}

\renewcommand{\shortauthors}{Chengyin Xu et al.}

\begin{abstract}
  Image retrieval has become an increasingly appealing technique with broad multimedia application prospects, where deep hashing serves as the dominant branch towards low storage and efficient retrieval. In this paper, we carried out in-depth investigations on metric learning in deep hashing for establishing a powerful metric space in multi-label scenarios, where the pair loss suffers high computational overhead and converge difficulty, while the proxy loss is theoretically incapable of expressing the profound label dependencies and exhibits conflicts in the constructed hypersphere space. To address the problems, we propose a novel metric learning framework with Hybrid Proxy-Pair Loss ({\name}) that constructs an expressive metric space with efficient training complexity {\wrt} the whole dataset. The proposed {\name} focuses on optimizing the hypersphere space by learnable proxies and excavating data-to-data correlations of irrelevant pairs, which integrates sufficient data correspondence of pair-based methods and high-efficiency of proxy-based methods. Extensive experiments on four standard multi-label benchmarks justify the proposed method outperforms the state-of-the-art, is robust among different hash bits and achieves significant performance gains with a faster, more stable convergence speed. Our code is available at \href{https://github.com/JerryXu0129/HyP2-Loss}{https://github.com/JerryXu0129/HyP2-Loss}.
\end{abstract}


\begin{teaserfigure}
  \centering
  \begin{overpic}[trim=0cm 0cm 0cm 0cm,clip,width=0.85\linewidth,grid=false]{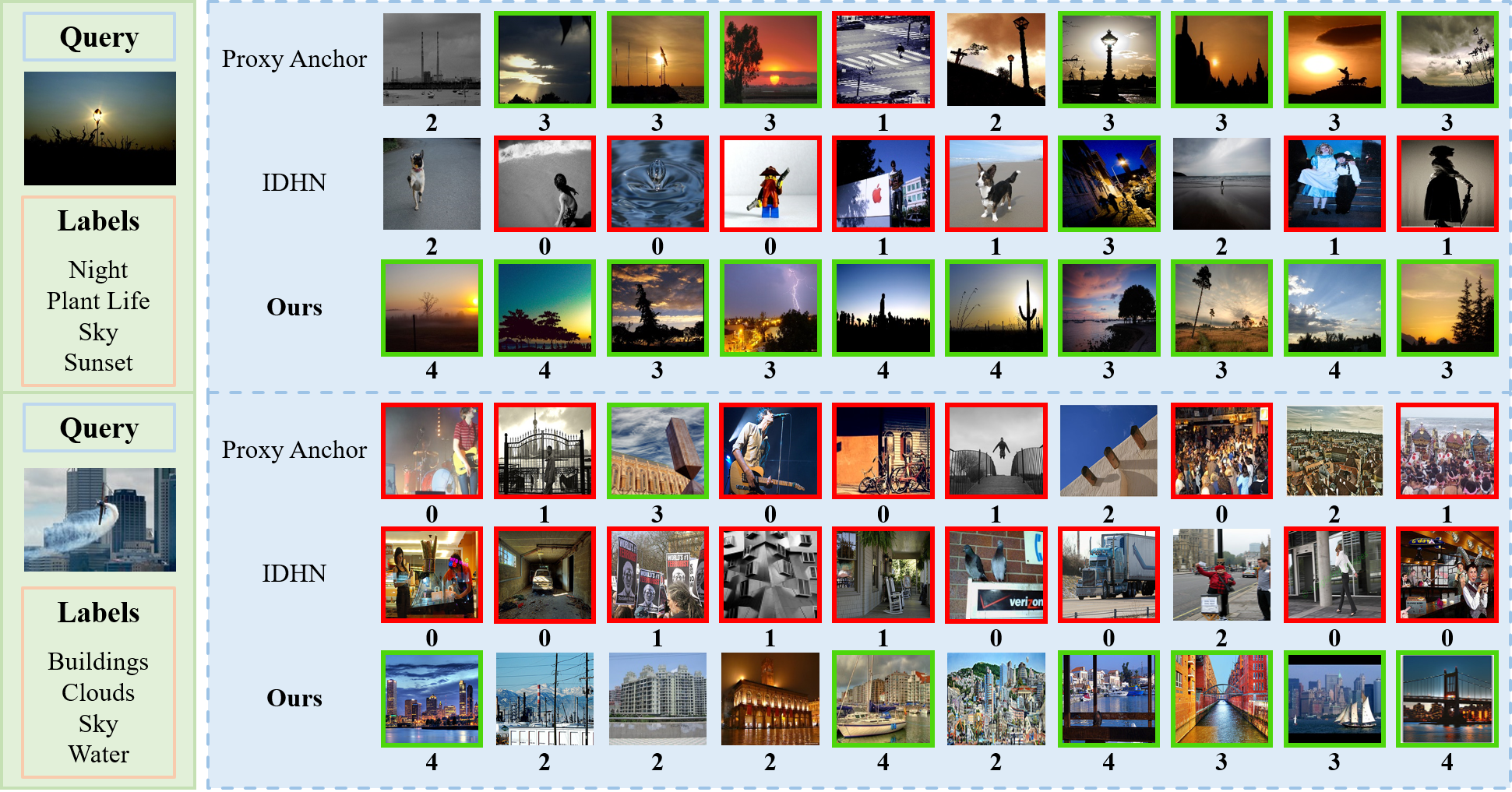}
  \put(18,46){\cite{Proxy_anchor_loss}}
  \put(18,38){\cite{IDHN}}
  \put(18,20){\cite{Proxy_anchor_loss}}
  \put(18,12){\cite{IDHN}}
  \end{overpic}
 \caption{Retrieval comparisons to previous methods~\cite{IDHN, Proxy_anchor_loss} $\bm{\&}$ ours on the Flickr-25k~\cite{flickr} dataset (top) $\bm{\&}$ NUS-WIDE~\cite{nuswide} dataset (bottom). The Top-$\bm{10}$ images are returned according to the Hamming distance between the query image and the database. The number below each retrieved image indicates the matched categories, green (red) box indicates matched categories $\Scale[1]{\bm{\geq75\%}}$  ($\Scale[1]{\bm{\leq25\%}}$). The proposed {\name} outperforms state-of-the-art~\cite{Proxy_anchor_loss,IDHN} with fewer contradictions and misclassified results.
 }
  \label{fig:teaser}
\end{teaserfigure}
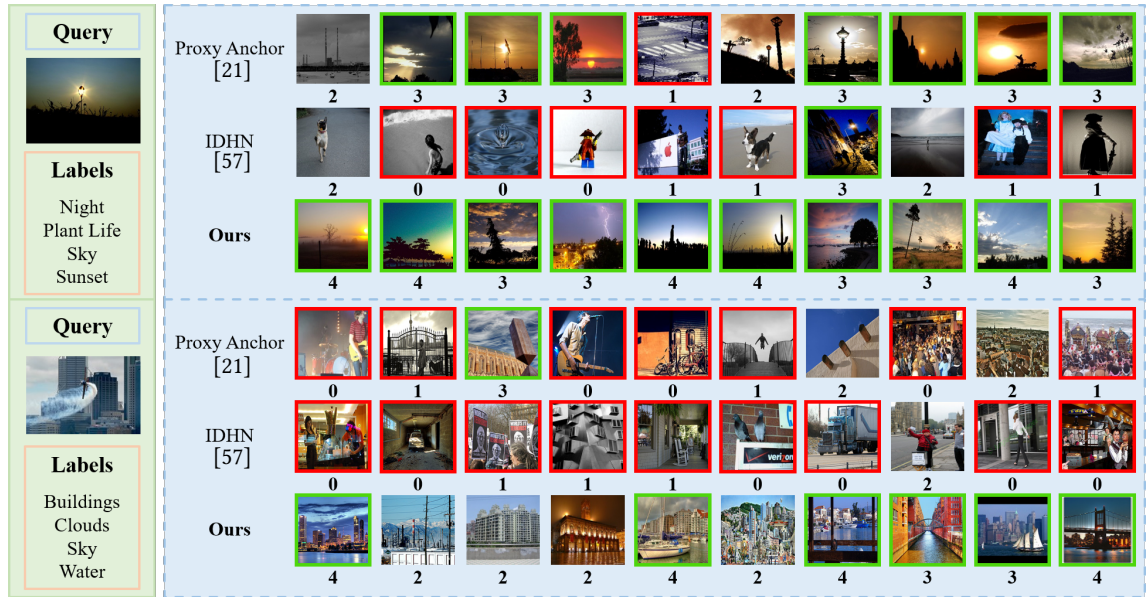

\maketitle

\section{Introduction}

The past decades have witnessed the arrival of the era of big data, floods of images are uploaded to social platforms and search engines day and night, which calls for more efficient and accurate image retrieval in multimedia applications~\cite{MM1, MM2, MM3, MM4, multi-label}. Real-world images typically contain more than one attribute. Hence, the multi-label retrieval task~\cite{multi_survey, multi-label} serves as the crucial and more challenging branch in large-scale image retrieval.

Hashing techniques~\cite{IR, IR2} are widely used to accelerate retrieval due to the low storage and computation costs. The retrieval system can utilize an efficient bit-wise XOR operation to estimate the distance between hash code pairs. Following the prosperous progress of Deep Neural Networks (DNNs) in visual recognition, deep hashing~\cite{CNNH, HashNet} has achieved glary attention and become one of the most substantial research topics in the image retrieval community~\cite{hashreview, deep, hashing}. The target of deep hashing is to project numerous samples into the hyper metric space and then convert them into compact binary codes through hash functions.
The parameterized networks are optimized such that semantically similar data (\ie, images with the same categories) are well clustered and distributed in the established metric space. Such quality and expressiveness of the metric space are optimized through elaborate-designed loss functions in a supervised manner during metric learning.

Pair-based methods~\cite{DSRH, NINH, OLAH, IAH, IDHN} are predominant in multi-label retrieval, which directly consider data-to-data connections in a mini-batch. However, such approaches prohibitively confront high training complexity that require square~\cite{IDHN,contrastive, contrastive2,smart,Signature94} or even higher~\cite{Triplet,N-pair,Lifted} complexity {\wrt} the number of training samples. Furthermore, the data-to-data correlations in a mini-batch could deteriorate the robustness and degrade the learned metric space because of the increasing overfitting risks and training instability.

Compared to the pair-based methods, proxy-based methods~\cite{Proxy_anchor_loss, Proxy-NCA} are more efficient and effective in single-label retrieval to embed samples into a proxy-centered hypersphere space. However, in multi-label scenarios, we observe and theoretically prove that they are limited to expressing profound correlations. With the exponentially increasing combination growth among labels, the inclusive ({\ie}, relevant categories) and exclusive ({\ie}, irrelevant categories) relations cannot get well established supervised by proxy loss. Hence, proxy-based methods exhibit conflicts (see Fig.~\ref{fig.method_illustrate}) and are unsatisfactory in such scenarios.

To overcome the above weaknesses,
we propose Hybrid Proxy-Pair Loss ({\name}) to embed samples into expressive hyperspace to establish abundant label correlations with an efficient training complexity. 
Concretely, we conceive the first part of {\name} by setting the learnable proxy for each category such that the established metric space roughly clusters samples of similar categories.
Note that simply adding the pair loss will introduce overwhelmed training complexity and is fruitless to performance gains. We creatively design additional irrelevant pair constraints as the second part to compensate for the missing multi-label data-to-data correlations, such that enables {\name} to alienate irrelevant samples to avoid attribute conflicts effectively.
Finally, we design the overall loss function with the learning algorithm that ensures more efficient training complexity than pair-based methods, since our predominant part during training is linear correlated to the total training dataset. The elaborate-designed loss functions and training framework guarantee that the established metric space contains expressive multi-label correlations.

To justify the effectiveness and efficiency of our framework, we conduct comprehensive experiments on four multi-label benchmark datasets, {\ie}, Flickr-25k~\cite{flickr}, VOC-2007~\cite{voc}, VOC-2012~\cite{voc}, and NUS-WIDE~\cite{nuswide}. Compared to existing state-of-the-art~\cite{OrthoHash, IDHN, CSQ, OLAH}, the proposed method outperforms existing techniques among different hash bits and backbones quantitatively and qualitatively with better convergence speed and stability. Additionally, in-depth ablation studies and visualization results justify the effectiveness of our mechanism and the designed loss function.

To summarize, our main contributions are three-fold:
\begin{itemize}[leftmargin=*,nosep,nolistsep]

\item We prove that the hypersphere metric space established by existing proxy-based methods is limited to expressing the profound inclusive and exclusive relations in multi-label scenarios. To the best of our knowledge, this is the first work to theoretically analyze the upper bound of distinguishable hypersphere number in metric space in multi-label image retrieval task.

\item We propose {\name}, a novel loss function that integrates the efficient time complexity of proxy-based methods and strong data correlations of pair-based methods. Particularly, the elaborate-designed multi-label proxy loss, irrelevant pair loss, and overall learning framework contribute well to embedding multi-label images into expressive metric space.

\item We conduct extensive experiments on four benchmarks to demonstrate the superiority and robustness of our proposed {\name}, which is also feasible to various deep hashing methods and backbones. {\name} exhibits its outperforming retrieval performance on both convergence speed and retrieval accuracy compared to the existing state-of-the-art.
\end{itemize}
\vspace{-2pt}

\section{Related Work}



Many hashing algorithms~\cite{IR, IR2, ITQ,SH,KSH,SDH} have been proposed to obtain compact binary codes, which are seminal solutions to reduce the storage and calculation overhead in large-scale image retrieval. Deep hashing~\cite{deep, hashing, CNNH, HashNet} has become the mainstream for the superiority of CNNs~\cite{lenet, alexnet, googlenet} in feature extraction, especially in scenarios where images are associated with more than one attribute. The common to all is to design the loss functions to establish powerful metric space and precise hashing positions. 
%



\myparagraph{Pair-based Methods.}
Pair-based methods~\cite{DSRH, IAH, DMSSPH, OLAH, IDHN, RCDH} are predominant in multi-label retrieval~\cite{multi_survey}, which focus on exploring data-to-data relations from the paired samples through metric learning. In the field of image retrieval, Constrictive loss~\cite{contrastive, contrastive2} innovatively determines the gradient descent directions by estimating the similarity between feature vector pairs. Based on it, CNNH~\cite{CNNH} and DPSH~\cite{DPSH} utilize CNNs to extract the features of given images. To reveal the local optima risks in pair-wise loss~\cite{DPSH}, Triplet Loss~\cite{Triplet, DTSH} associates the anchor with one positive and one negative sample for the loss calculation process. 

Recently, researchers concentrate on exploring the profound attribute correlations in challenging multi-label retrieval~\cite{multi_survey}. The seminal work DSRH~\cite{DSRH} introduces CNN-based Triplet Loss to estimate the semantic distance according to the sorted labels. IAH~\cite{IAH} divides hash codes into groups to separately excavate instance-aware image representations. DMSSPH~\cite{DMSSPH} and RCDH~\cite{RCDH} further improve the performance by considering the semantic similarity by supervision on grouped labels and additional regularization.
%
\indent Pair-based methods fully excavate the data correlations and exhibit satisfactory performance. However, the training complexity generally requires square or cubic complexity related to the entire large-scale images. Hence, such approaches suffer high computational consumption and converge difficulty, especially are more serious and ineluctable in multi-label scenarios.

\myparagraph{Proxy-based Methods.}
To address the challenging issues in pair-based methods, proxy-based methods are proposed to improve model robustness with efficient training complexity in single-label scenarios. Some methods~\cite{CSQ, OrthoHash, DPN} attempt to alleviate the training difficulty by fixing manually-selected or predefined hash centers. Such predefined centers are regarded as specific proxies for corresponding categories. Hence, the training complexity is affordable because each sample only interacts with a few class proxies.

However, the artificially designed proxies ignore the semantic relationship between intra-class and inter-class. To fill this gap, existing state-of-the-art~\cite{Proxy-NCA,Manifold,Proxy_anchor_loss} regard the hash centers as trainable parameters. Proxy NCA~\cite{Proxy-NCA} calculates the distance between each proxy and positive \& negative samples, while Manifold Proxy Loss~\cite{Manifold} improves performance with the manifold-aware distance to measure the semantic similarity. Proxy Anchor Loss~\cite{Proxy_anchor_loss} integrates both advantages of pair-based and proxy-based schemes with \textit{Log-SumExp} function. It individually considers the distances between different samples and proxies to tackle the hard-pair challenges.

Although the proxy-based methods achieve performance better or on par with pair-based ones with faster convergence speed and promising training overhead in terms of single-label datasets~\cite{cifar, ImageNet}, they always fail when come into multi-label scenarios and hence haven't been fully investigated before. We further elaborate on the reasons and propose our novel {\name} solution further.

\section{Methodology}

\subsection{Task Definition}

Given a training set $\mathscr{D}_M\coloneqq\{({\bm{x}}_i,\bm{y}_i)\}_{i=1}^M$ composed of $M$ data points $\mathcal{X}\coloneqq\{\bm{x}_i\}_{i=1}^M \in \mathds{R}^{D\times M}$ and corresponding label $\mathcal{Y}\coloneqq\{\bm{y}_i\}_{i=1}^M \in \{0,1\}^{C\times M}$, where $D$ represents the resolution of images and $C$ denotes the category numbers, respectively. The image $\bm{x}_i$ contains the attribute of class $y_j$ iff ${\bm{y}_i}_{(j)}=1$. In multi-label scenarios, each sample contains at least one attribute, {\ie}, $\sum\nolimits_{j = 1}^C {{\bm{y_i}}_{(j)}}  \ge 1$.
The target for deep hashing is to learn a feature extractor $\mathcal{F}$ parameterized by $\Theta$ that encodes each data point $\bm{x}_i$ into a compact $K$-bit feature vector $\bm{v}_i \coloneqq \mathcal{F}_{\Theta}(\bm{x}_i) \in \mathds{R}^{K}$ in metric space, and maps into $K$-bit binary hash code $\bm{b}_i \coloneqq \mathcal{H}(\bm{v}_i) \in \{-1, 1\}^K$ through the hashing function $\mathcal{H}$ in the Hamming space. 
Hence, the image-wise similarity is preserved in the Hamming space.
For given query image $\bm{x}_q$, we sort the hash codes for all the samples in the database according to their Hamming distance, and return the Top-$N$ images as the query results.
The core challenge of this task is to learn a reliable feature extractor $\mathcal{F}_{\Theta}^*$ to cluster images of different categories with proper and distinguishable hash positions.

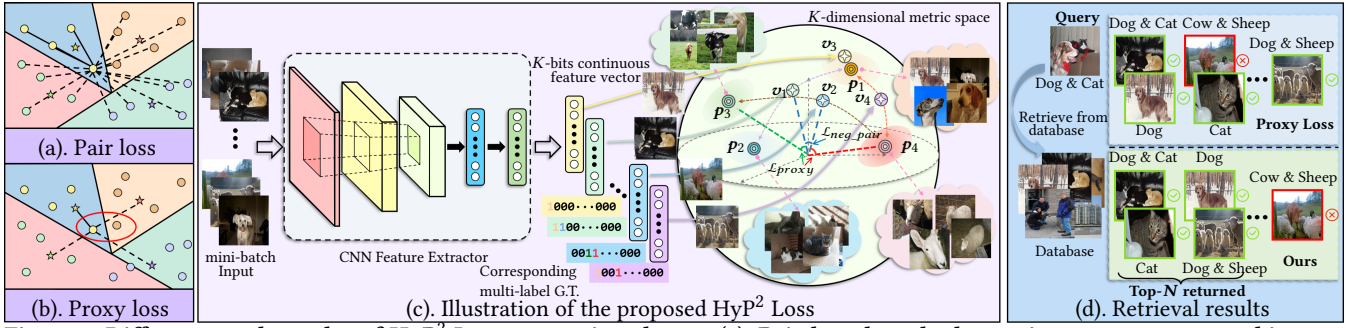
\begin{figure*}[t!]
    \centering
    \begin{overpic}[trim=0cm 0cm 0cm 0cm,clip,width=1\linewidth,grid=false]{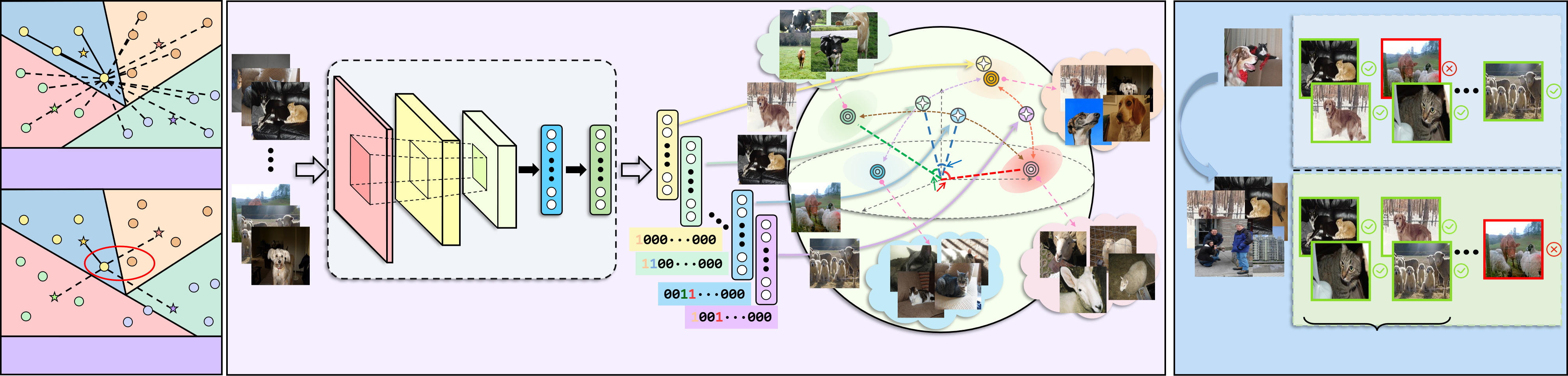}
    \put(2.5,12.5){(a). Pair loss}
    \put(2,0.5){(b). Proxy loss}
    \put(30,0.5){(c). Illustration of the proposed {\name}}
    \put(80,0.5){(d). Retrieval results}

    \put(57,11.3){\tiny {$\Scale[0.8]{\mathcal{L}_{proxy}}$}}
    \put(61,14){\tiny {$\Scale[0.8]{\mathcal{L}_{neg\_pair}}$}}
    
    \put(25,4.5){\scriptsize CNN Feature Extractor}
    \put(15,4.5){\scriptsize mini-batch}
    \put(16,3.5){\scriptsize Input}
    
    \put(35.5,3.5){\scriptsize Corresponding }
    \put(35.5,2){\scriptsize multi-label G.T.}
    
    \put(39.4,19){\scriptsize $K$-bits continuous}
    \put(40.8,18){\scriptsize feature vector}

    \put(60,22.5){\scriptsize $K$-dimensional metric space}

    \put(53,15.5){\scriptsize $\bm{p}_3$}
    \put(67,13){\scriptsize $\bm{p}_4$}    
    \put(54,13){\scriptsize $\bm{p}_2$}
    \put(63,17.5){\scriptsize $\bm{p}_1$}

    \put(57.3,17){\scriptsize $\bm{v}_1$}
    \put(61,17.3){\scriptsize $\bm{v}_2$}    

    \put(61,20.5){\scriptsize $\bm{v}_3$}
    \put(63.5,16.5){\scriptsize $\bm{v}_4$}   
    
    \put(77,5){\scriptsize Database}
    
    \put(75.8,15){\scriptsize Retrieve from}
    \put(76.5,14){\scriptsize database}
    
    \put(78.5,22.5){\scriptsize \textbf{Query}}
    \put(77,17.5){\scriptsize Dog \& Cat}

    \put(82.5,22){\scriptsize Dog \& Cat}
    \put(88,22){\scriptsize Cow \& Sheep}
    \put(93,20.5){\scriptsize Dog \& Sheep}
    \put(84.5,14){\scriptsize Dog}
    \put(90,14){\scriptsize Cat}
    \put(93.5,14.5){\scriptsize \textbf{Proxy Loss}}
    
    \put(82.5,12){\scriptsize Dog \& Cat}
    \put(89,12){\scriptsize Dog}
    \put(93,10.5){\scriptsize Cow \& Sheep}
    \put(84.5,3.7){\scriptsize Cat}
    \put(88,3.7){\scriptsize Dog \& Sheep}
    \put(95.5,4.2){\scriptsize \textbf{Ours}}

    \put(84,2){\scriptsize \textbf{Top-$\bm{N}$ returned}}
    
    \end{overpic}
    \caption{Difference and novelty of {\name} {\vs} previous losses. (a). Pair-based methods require square or even cubic complexity {\wrt} the whole dataset. (b). Proxy-based methods are efficient that only consider proxy-to-data relations but ignore data-to-data relation and exhibits conflicts in multi-label retrieval. (c). The novel {\name} extends proxy loss into multi-label scenarios to construct powerful metric space beyond hypersphere ones. The Multi-label Proxy Loss $\bm{\mathcal{L}_{proxy}}$ establishes a proper metric space with training efficiency guaranteed, while the elaborate-designed irrelevant pair constraint $\bm{\mathcal{L}_{neg\_pair}}$ alienates irrelevant pairs to address conflicts between $\bm{(v}_1,\bm{v}_2\bm{)}$. (d). Given the query, proxy-based methods misclassify the irrelevant samples into the Top-$\bm{N}$ results. As a comparison, {\name} achieves more accurate top returned retrieval results.}
    \vspace{-12pt}
    \label{fig.method_illustrate}
\end{figure*}

\subsection{Motivation} \label{Sec.motivation}

\myparagraph{Upper Bound of Distinguishable Hypersphere Number in Metric Space.}
Metric learning serves as the substantial procedure for deep hashing. It focuses on embedding the 2D images that consist of various attributes into the $K$-dimensional metric space (\ie, through their feature vectors $\bm{v}$), where similar samples (\ie, with closer categories) should get clustered.

For image retrieval, it is challenging to establish precise mapping that encodes the input images into the ideal metric space, especially in large-scale datasets~\cite{ImageNet} that should consider training complexity. 

In single-label retrieval (\eg, ImageNet~\cite{ImageNet}, CIFAR~\cite{cifar}), the label correlations are restricted among one positive label with negative others. Hence, proxy-based methods~\cite{Proxy-NCA, Manifold, Proxy_anchor_loss} are effective to embed samples around the class proxies in the metric space. Ultimately, samples tend to distribute in a hypersphere centered by predefined or learnable proxies, where similar proxies are close while heterogeneous others are well alienated.
While in multi-label retrieval, the inclusive and exclusive relations ({\ie}, the ability of metric space to distinguish its relevant and irrelevant categories) is exponential {\wrt} category number $C$.

However, the isotropic ({{\ie}, perfectly symmetrical}) hypersphere has inherent side-effect in such scenarios. When the labels of one sample larger than two, the label correlations cannot get fully expressed in the hypersphere space. Specifically, we have the following Theorem.~\ref{theo:circle} to illustrate the upper bound of distinguishable hypersphere number $\Omega(K,C)$ (or called the maximum inclusive and exclusive relations) in multi-label scenarios. Please refer to the supplementary for proof.

\begin{theorem} 
\label{theo:circle}
For the $K$-dimensional metric space $\mathds{R}^K$ with $C$ hypersphere $\mathds{S} \in \mathds{R}^{K}$. The upper bound of distinguishable hypersphere number $\Omega(K,C)$ cannot enumerate the ideal $ \Omega^*(K,C)=\sum\nolimits_{c = 0}^C \tbinom{C}{c} =2^C$ when $C>K+1$. The upper bound is limited at:
\begin{equation}
\tilde{\Omega}(K,C) = \mathop {\sup}\nolimits_{\mathds{S}} {{\Omega}(K,C)}=\tbinom{C-1}{K}+\sum\nolimits_{k=0}^K \tbinom{C}{k}<2^C
\end{equation}
\end{theorem}

\myparagraph{Embedding Position Conflicts in Multi-label Scenarios.}
Another limitation of proxy-based methods in multi-label scenarios is that, some irrelevant samples ({\ie}, without the same categories) are inevitable to be embedded into nearby positions. The primary reason is that the proxy loss only considers the proxy-to-data distance but misses the data-to-data constraints, the irrelevant samples associated with different attributes will be potentially encoded into the close positions nearby the middle of category proxies.

The proxy loss will enforce samples with multiple attributes embedded into the middle among these proxies, because such hash position
ensures images with identical multi-label retrieved by query images in priority, second by images containing partially same attributes. Hence, proxy loss encourages multi-label samples to converge nearby the middle of proxies to achieve the optimal solution.
Intuitively, as illustrated in Fig.~\ref{fig.method_illustrate}{(c)}, suppose the proxy set $\mathscr{P}_4=\{\bm{p}_1,\cdots,\bm{p}_4\}$ is well embedded into the metric space, where the proxy loss between any sample $(\bm{x},\bm{y})$ and $\mathscr{P}$ has converged. 
Hence, image $\bm{x}_1$ associated with $y_1,y_2$ ({\eg, dog \& cat}) will be embedded into the middle of $\bm{p}_1$ and $\bm{p}_2$, while image $\bm{x}_2$ with $y_3,y_4$ ({\eg, cow \& sheep}) will be embedded into the middle of $\bm{p}_3$ and $\bm{p}_4$. However, the missing data-to-data correlations ignore the attribute conflicts between the irrelevant $\bm{x}_1$ and $\bm{x}_2$. Although proxy loss explicitly alienates $\bm{x}_1$ and $\bm{p}_3,\bm{p}_4$, $\bm{x}_1$ is not guaranteed away from the middle of $\bm{p}_3$ and $\bm{p}_4$, which is exactly the embedding position of $\bm{x}_2$. Hence, the dissimilar samples $\bm{x}_1$ and $\bm{x}_2$ are entangled in the metric space. For given query $\bm{x}_q$ with attributes $y_1,y_2$, $\bm{x}_1$ may get retrieved first, second by the closest sample $\bm{x}_2$, but ignores some relevant samples $(\bm{x}_3,y_1)$ or $(\bm{x}_4,\{y_1,y_4\})$.
The misclassified conflicts among multi-label datasets become more prominent. See Fig.~\ref{fig:teaser} for some examples, the proxy-based method wrongly retrieves the results of given multi-label images.

\subsection{Hybrid Proxy-Pair Loss}

Sec.~\ref{Sec.motivation} reveals the primary reasons that proxy-based methods are unsatisfactory in multi-label scenarios, {\ie}, simply embedding samples distributed among a proxy-centered hypersphere cannot comprehensively introduce the combination among various categories, and the proxy-to-data supervisions ignore data-to-data attribute conflicts. As a result, some crucial label correlations may not be well-expressed, especially under large-scale datasets with limited $K$-bit hash codes.

The above observations motivate us to consider the data-to-data relations that contribute to a powerful metric space to represent the correlations among various attributes. 
To avoid constructing the metric space into an isotropic hypersphere without loss of training efficiency, we creatively propose Hybrid Proxy-Pair Loss ({\name}) for metric learning to extend the proxy loss into challenging multi-label scenarios, and compensate for the local optimum and overfitting risks of pair loss in exploring data-to-data relations.
The carefully designed {\name} depicts a superior metric space to fully express the profound label correspondences. The overview framework of {\name} is illustrated in Fig.~\ref{fig.method_illustrate}, and the details of each component are elaborated as follows.

\myparagraph{Multi-label Proxy Loss.}
Firstly, we set $C$ learnable proxies $\mathscr{P}_C=\{\bm{p}_1,\cdots,\bm{p}_C\}$, each $\bm{p}_i\in \mathscr{P}_C$ is a compact $K$-bit vector that is exclusive for each category. For a given feature vector $\bm{v}_i\coloneqq \mathcal{F}_{\Theta}(\bm{x}_i)$ and corresponding label $\bm{y}_i$, the energy term between any $\bm{v}_i$ and $\bm{p}_j$ is $\cos_{+}(\bm{v}_i,\bm{p}_j) \coloneqq -\cos\langle\bm{v}_i,\bm{p}_j\rangle=-\frac{|\bm{v}_i\cdot\bm{p}_j|}{|\bm{v}_i|\cdot|\bm{p}_j|}$ iff they are a positive pair, {\ie}, ${\bm{y}_i}_{(j)}=1$. Otherwise, $\bm{v}_i$ and $\bm{p}_j$ are a negative pair. The energy term is defined as $\cos_{-}(\bm{v}_i,\bm{p}_j)\coloneqq \left(\cos \langle\bm{v}_i,\bm{p}_j\rangle-\zeta\right)_+=\max \left(\frac{|\bm{v}_i\cdot \bm{p}_j|}{|\bm{v}_i|\cdot|\bm{p}_j|} - \zeta,0\right)$, where $\zeta=\zeta(C,K)$ is a margin term that follows HHF~\cite{HHF}.
Then, the first term of {\name} is designed as Multi-label Proxy Loss $\mathcal{L}_{proxy}$, which only optimizes the distance between proxies and samples, as Eq.~\ref{Eq.loss} illustrates.
\begin{equation}
\begin{aligned}
\mathcal{L}_{proxy}(\mathscr{D}_M,\mathscr{P}_C)&= 
\frac{{\sum\nolimits_{i = 1}^M {\sum\nolimits_{j = 1}^C {\mathds{1}({\bm{y}_i}_{(j)} = 1){{\cos }_ + }(\bm{v}_i,\bm{p}_j)} } }}{{\sum\nolimits_{i = 1}^M {\sum\nolimits_{j = 1}^C {\mathds{1}({\bm{y}_i}_{(j)} = 1)} } }} \\&+ \frac{{\sum\nolimits_{i = 1}^M {\sum\nolimits_{j = 1}^C {\mathds{1}({\bm{y}_i}_{(j)} = 0){{\cos }_{-} }({\bm{v}_i},{\bm{p}_j})} } }}{{\sum\nolimits_{i = 1}^M {\sum\nolimits_{j = 1}^C {\mathds{1}({\bm{y}_i}_{(j)} = 0)} } }}
\end{aligned},
\label{Eq.loss}
\end{equation}
where $\mathds{1}(\cdot)$ is the indicator function that equals to $1$ ($0$) iff $(\cdot)$ is True (False). The denominator term balances $\cos_+(\cdot)$ and $\cos_-(\cdot)$, such that avoids the gradient bias introduced by overmuch negative pairs.
The Multi-label Proxy Loss is used for establishing a primary metric space to ensure the samples are distributed among the cluster centers ({{\cf} Fig.~\ref{fig.tsne}}), where the correlated labels are properly clustered and irrelevant sample-proxy pairs are roughly alienated.

\myparagraph{Irrelevant Pair Loss.} 
Secondly, we focus on exploring data-to-data correlations to explicitly enforce irrelevant samples get alienated. 
To achieve this, we define the irrelevant pairs as: $(\bm{v}_i,\bm{v}_j)$ associated with label $(\bm{y}_i,\bm{y}_j)$ is irrelevant pairs iff $|\bm{y}_i\cdot \bm{y}_j|=0$ and $|\bm{y}_i|>1,|\bm{y}_j|>1$. Note that the total number of such irrelevant pairs is far fewer than the total $M\times M$ pairs, the ratio is defined as $\eta$ in Tab.~\ref{Tab:dataset}.
Suppose the subset $\mathscr{D}_{M'}'\coloneqq\{({\bm{x}}_{i},\bm{y}_{i})\}_{i=1}^{M'} \subseteq \mathscr{D}_M$ is composed of $M'$ samples, where each sample contains more than one category. Then the second term of {\name} is Irrelevant Pair Loss $\mathcal{L}_{neg\_pair}$, which is defined as Eq.~\ref{Eq.negpair}.
\begin{equation}
    \mathcal{L}_{neg\_pair}(\mathscr{D}_{M'}') = \frac{\sum\nolimits_{i=1}^{M'}\sum\nolimits_{j=1}^{M'} \mathds{1}(|\bm{y}_i\cdot \bm{y}_j|=0) \cos_-(\bm{v}_{i}, \bm{v}_{j})}{\sum\nolimits_{i=1}^{M'}\sum\nolimits_{j=1}^{M'} \mathds{1}(|\bm{y}_i\cdot \bm{y}_j|=0)},
    \label{Eq.negpair}
\end{equation}
where $\cos_-(\bm{v}_{i}, \bm{v}_{j})$ indicates the pair-wise similarity of given irrelevant samples. Compared to the entire $M\times M$ computation, the proposed $\mathcal{L}_{neg\_pair}(\mathscr{D}_{M'}')$ only considers limited pairs to mine data-to-data correlations that alienates irrelevant samples effectively without loss of efficiency.

\begin{algorithm}[t!]
    \caption{The training algorithm of the proposed {\name} for metric learning.}
    \label{Alg:training_manner}
	\renewcommand{\algorithmicrequire}{\textbf{Input:}}
	\renewcommand{\algorithmicensure}{\textbf{Output:}}
	\begin{algorithmic}[1]
		\REQUIRE Training dataset $\Scale[0.95]{\mathscr{D}_M}$, hash code length $K$, mini-batch $B$.
		\ENSURE Optimized network $\Scale[0.95]{\mathcal{F}_{\Theta}^*}$ and proxy set $\Scale[0.95]{\mathscr{P}^*_C}$.
		\STATE Initialize $\Scale[0.95]{\Theta \leftarrow \Theta^{(0)}}$, $\Scale[0.95]{\mathscr{P} \leftarrow \mathscr{P}^{(0)}}$, Epoch $T \leftarrow 0$;
		\REPEAT
	    \STATE Randomly sample a mini-batch samples $\Scale[0.95]{\mathscr{D}_{B}\coloneqq \{(\bm{x}_i,\bm{y}_i)\}_{i=1}^{B}}$;
		\STATE Compute feature vector $\Scale[0.95]{\bm{v}_i \coloneqq \mathcal{F}_{\Theta}^{(T)}(\bm{x}_i)}$ for each $(\bm{x}_i,\bm{y}_i) \in \mathscr{D}_{B}$ by forward propagation;
		\STATE Compute Multi-label Proxy Loss $\Scale[0.85]{\mathcal{L}_{proxy}(\mathscr{D}_{B},\mathscr{P}^{(T)}_C)}$ via Eq.~\ref{Eq.loss};
		\STATE Compute Irrelevant Pair Loss $\Scale[0.85]{\mathcal{L}_{neg\_pair}(\mathscr{D}_{B'}')}$ via Eq.~\ref{Eq.negpair};
		\STATE Compute Total Loss $\Scale[0.85]{\mathcal{L}_{total}(\mathscr{D}_{B},\mathscr{P}^{(T)}_C)}$ via Eq.~\ref{Eq.allloss};
		\STATE Compute gradient $\frac{{\partial \mathcal{L}_{total}}}{{\partial{\cos\langle\bm{v}_i,\bm{p}_j\rangle}}}$ and $\frac{{\partial \mathcal{L}_{total}}}{{\partial{\cos\langle\bm{v}_{i},\bm{v}_{j}\rangle}}}$ via Eq.~\ref{Eq.grad};
	    \STATE Update $\Scale[0.95]{\Theta^{(T+1)}}$ and $\Scale[0.95]{\mathscr{P}^{(T+1)}}$ by back propagation;
		\STATE $T\leftarrow T+1$;
    	\UNTIL{Convergence}
    	\STATE Return $\Scale[0.95]{\Theta^{*} \leftarrow \Theta^{(T)}}$, $\Scale[0.95]{\mathscr{P}^{*} \leftarrow \mathscr{P}^{(T)}}$.
    
	\end{algorithmic}
\end{algorithm}
\setlength{\textfloatsep}{0pt}

\myparagraph{Overall Loss \& Gradient of {\name}.}
Finally, the overall {\name} is the weighted-assumption of the above two loss terms to obtain $\mathcal{L}_{total}$, as Eq.~\ref{Eq.allloss} illustrates.
\begin{equation}
    \mathcal{L}_{total}(\mathscr{D}_M,\mathscr{P}_C) = \mathcal{L}_{proxy}(\mathscr{D}_M,\mathscr{P}_C) + \beta \mathcal{L}_{neg\_pair}(\mathscr{D}_{M'}')
    \label{Eq.allloss}
\end{equation}
where $\beta$ is a hyperparameter to balance the constraints between multi-label proxy term and irrelevant pair term.

Hence, to optimize the parameterized network $\mathcal{F}_{\Theta}$ and proxy set $\mathscr{P}_C$, the objective function is to minimize $\mathcal{L}_{total}$ of the given training set and learnable proxy set, as Eq.~\ref{Eq.obj_func} illustrates.
\begin{equation}
    \Theta^*,\mathscr{P}^*=\mathop {\arg\min}_{\Theta,\mathscr{P}} \mathcal{L}_{total}(\mathscr{D}_M,\mathscr{P}_C)
    \label{Eq.obj_func}
\end{equation}

To achieve this, the gradient of {\name} in Eq.~\ref{Eq.allloss} {\wrt} $\cos\langle\bm{v}_i,\bm{p}_j\rangle$ and $\cos\langle\bm{v}_i,\bm{v}_j\rangle$ is given by Eq.~\ref{Eq.grad}.
\begin{equation}
\begin{aligned}
\frac{{\partial {\mathcal{L}_{total}}}}{{\partial{\cos\langle\bm{v}_i,\bm{p}_j\rangle}}} &= 
\left\{\begin{array}{ll}
 -\frac{1}{\sum\nolimits_{i = 1}^M {\sum\nolimits_{j = 1}^C {\mathds{1}({\bm{y}_i}_{(j)} = 1)}}}, & {\bm{y}_i}_{(j)} = 1\\
 \frac{1}{\sum\nolimits_{i = 1}^M {\sum\nolimits_{j = 1}^C {\mathds{1}({\bm{y}_i}_{(j)} = 0)}}}, & {\bm{y}_i}_{(j)} = 0, \cos\langle\bm{v}_i,\bm{p}_j\rangle > \zeta\\
 0, & {\bm{y}_i}_{(j)} = 0, \cos\langle\bm{v}_i,\bm{p}_j\rangle \leq \zeta
\end{array} \right.
\\
\frac{{\partial {\mathcal{L}_{total}}}}{{\partial{\cos\langle\bm{v}_{i},\bm{v}_{j}\rangle}}} &=
\left\{\begin{array}{ll} 
\frac{\beta}{\sum\nolimits_{i = 1}^{M'} {\sum\nolimits_{j = 1}^{M'} {\mathds{1}(|\bm{y}_i\cdot \bm{y}_j|=0)}}}, & |\bm{y}_i\cdot \bm{y}_j|=0, \cos\langle\bm{v}_{i},\bm{v}_{j}\rangle > \zeta\\
0, & |\bm{y}_i\cdot \bm{y}_j|=0,  \cos\langle\bm{v}_{i},\bm{v}_{j}\rangle \leq \zeta 
 \end{array} \right.
\end{aligned}
\label{Eq.grad}
\end{equation}

Eq.~\ref{Eq.grad} shows that minimizing the {\name} enforces $\bm{v}_i$ and $\bm{p}_j$ to get close if the two share the same attributes, and distinguishes the irrelevant proxy-to-data/data-to-data pairs simultaneously. When {\name} convergences, we thus construct the powerful metric space by mapping the images from the database into continuous feature vectors, and binarizing into hash codes in the Hamming space for efficient retrieval.

\subsection{Overview of Learning Algorithm}

\myparagraph{Training Algorithm.}
With the novel {\name}, we ensure the constructed metric space is more powerful than existing proxy-based methods~\cite{Proxy-NCA, Manifold, Proxy_anchor_loss} both theoretically and experimentally, because {\name} explicitly enforces the established metric space considers the data-to-data correspondences that tackles the conflicts effectively.
To achieve this, we present the training algorithm in Algo.~\ref{Alg:training_manner}. During the training process, the standard back-propagation algorithm~\cite{back} with mini-batch gradient descent method is used to optimize the network.

\begin{table}[t!]
\setlength{\tabcolsep}{12pt}
\centering
\caption{\textbf{Training complexity of {HyP$\Scale[0.9]{\bm{^2}}$} \vs previous state-of-the-art.} Note that proxy-based methods are unqualified for multi-label metric learning, while pair-based methods require $\Scale[0.9]{\bm{\mathcal{O}(M^2)}}$ or even $\Scale[0.9]{\bm{\mathcal{O}(M^3)}}$ time complexity. As a comparison, HyP$\Scale[0.9]{\bm{^2}}$ costs more efficient training complexity since the irrelevant pairs are the minority to the whole dataset.}
\resizebox{1\linewidth}{!}{%
\scriptsize
\begin{tabular}{l|l|l}
\toprule
Type                   & Method             & Time Complexity                              \\ \midrule
\multirow{3}{*}{Proxy} & Proxy NCA~\cite{Proxy-NCA}         & $\mathcal{O}(MC)$                                        \\
                       & Proxy Anchor~\cite{Proxy_anchor_loss}         & $\mathcal{O}(MC)$                                        \\
                        & OrthoHash~\cite{OrthoHash}       & $\mathcal{O}(MC)$ \\
                       & SoftTriple~\cite{soft}       & $\mathcal{O}(MCk^2)$                    \\
                                         \midrule
\multirow{8}{*}{Pair}  & Constrastive~\cite{contrastive, contrastive2,Signature94}      & $\mathcal{O}(M^2)$                      \\               &HashNet~\cite{HashNet}&$\mathcal{O}(M^2)$\\
                       & DHN~\cite{DHN}&$\mathcal{O}(M^2)$\\
                       & IDHN~\cite{IDHN}&$\mathcal{O}(M^2)$\\
                       & Triplet (Smart)~\cite{smart}   & $\mathcal{O}(M^2)$                    \\
                      & Triplet (Semi-Hard)~\cite{Triplet} & $\mathcal{O}(M^3/B^2)$ \\
                       & $N$-pair~\cite{N-pair}           & $\mathcal{O}(M^3)$                     \\
                       & Lifted Structure~\cite{Lifted}   & $\mathcal{O}(M^3)$                     \\ \midrule
\textbf{Ours}                   & {\name}                & $\mathcal{O}(MC + \eta M^2)$                                     \\ 
\bottomrule
\end{tabular}%
}
\label{Tab:time_comp}
\vspace{-2pt}
\end{table}
\setlength{\textfloatsep}{0pt}

\begin{table}[t!]
\caption{Statistics of four benchmarks, where $\bm{\eta}$ indicates the ratio of irrelevant sample pairs with multiple labels to all sample pairs in the dataset.}
\setlength{\tabcolsep}{3pt}
\resizebox{1\linewidth}{!}{%
\scriptsize
\begin{tabular}{l|cccccc}
\toprule
Datasets  & \# Dataset & $C$ & \# Database & \# Train & \# Query  & $\eta$ \\ \midrule
Flickr-25k & 24,581     & 38            & 19,581      & 4,000       & 1,000   & 0.286 \\
NUS-WIDE  &  195,834    & 21            & 183,234     & 10,500      & 2,100   & 0.242 \\
VOC-2007   & 9,963      & 20            & 5,011       & 5,011       & 4,952   & 0.062 \\
VOC-2012   & 11,540     & 20            & 5,717       & 5,717       & 5,823    & 0.055\\ 
\bottomrule
\end{tabular}%
}
\label{Tab:dataset}
\end{table}
\setlength{\textfloatsep}{0pt}

\myparagraph{Time Complexity Analysis.}
The proposed method converges faster and is proven more efficient and stable than those pair-based methods~\cite{DSRH, IAH, OLAH, RCDH, IDHN} (as we will justify in Fig.~\ref{fig.converge}). Below we analyze the training complexity of {\name}.
Note that $M$, $C$, $B$, and $k$ denote the training sample number, category number, mini-batch size, and the proxy number of each category, respectively. $\eta$ is specifically defined in {\name}, which indicates the ratio of irrelevant sample pairs with multiple labels to all $M\times M$ pairs in the dataset. We omit $k \equiv 1$ in single-proxy methods~\cite{Proxy-NCA,Proxy_anchor_loss} and ours for simplicity. $k$ is nontrivial for managing multiple proxies per class such as SoftTriple Loss~\cite{soft}.

Tab.~\ref{Tab:time_comp} comprehensively compares the training complexity of {\name} (ours) to state-of-the-art pair-based and proxy-based methods. The complexity of {\name} is $\mathcal{O}(MC + \eta M^2)$ since it compares each sample with positive or negative proxies and its irrelevant samples (if exists) in a mini-batch.
More specifically, in Eq.~\ref{Eq.allloss}, the complexity of the first summation requires $M_+$ ($M_-$) times calculation for positive (negative) proxy-to-data pairs, respectively. Hence the total training complexity is $\mathcal{O}(M_+C + M_-C)=\mathcal{O}(MC)$. Then the second term requires $\eta M^2$ times calculation for each irrelevant pair. The first term of {\name} is linear correlated to $M$ and $C$, while the second term is significantly degraded because $\eta$ is much less than $1$ in general.

\section{Experiment}

\begin{table*}[t!]
\centering
\caption{mAP performance by Hamming Ranking for different hash bits ($\bm{K\in\{12,24,36,48\}}$) in Flickr-25k (mAP@$\bm{1,000}$) and NUS-WIDE (mAP@$\bm{5,000}$) with AlexNet. $^\ast$: reported results with the same experiment settings from~\cite{IDHN}. $^\dagger$: our reproduced results through the publicly available models. Bold font (underlined) values indicate the best (second best).}
\resizebox{1\linewidth}{!}{%
\scriptsize
\begin{tabular}{c|c|ccccc|ccccc}
\toprule
\multirow{2}{*}{Method} & Dataset    & \multicolumn{5}{c|}{Flickr-25k}                                                                         & \multicolumn{5}{c}{NUS-WIDE}                                                                   \\ \cmidrule{2-12} 
                        & Pub.       & 12             & 24             & 36             & \multicolumn{1}{c|}{48}             & avg. $\Delta$          & 12             & 24             & 36             & \multicolumn{1}{c|}{48}             & avg. $\Delta$ \\ \midrule
DLBHC~\cite{DLBHC}$^\ast$                   & CVPR"15   & 0.724          & 0.757          & 0.757          & \multicolumn{1}{c|}{0.776}          & -              & 0.570          & 0.616          & 0.621          & \multicolumn{1}{c|}{0.635}          & -     \\
DQN~\cite{DQN}$^\ast$                     & AAAI"16   & 0.809          & 0.823          & 0.830          & \multicolumn{1}{c|}{0.827}          & 0.069          & 0.711          & 0.733          & 0.745          & \multicolumn{1}{c|}{0.749}          & 0.124 \\
DHN~\cite{DHN}$^\dagger$                     & AAAI"16   & 0.817          & 0.831          & 0.829          & \multicolumn{1}{c|}{0.851}          & 0.079          & 0.720          & 0.742          & 0.741          & \multicolumn{1}{c|}{0.749}          & 0.128 \\
HashNet~\cite{HashNet}$^\ast$                 & ICCV"17   & 0.791          & 0.826          & 0.841          & \multicolumn{1}{c|}{0.848}          & 0.073          & 0.643          & 0.694          & 0.737          & \multicolumn{1}{c|}{0.750}          & 0.095 \\
DMSSPH~\cite{DMSSPH}$^\ast$                  & ICMR"17   & 0.780          & 0.808          & 0.810          & \multicolumn{1}{c|}{0.816}          & 0.050          & 0.671          & 0.699          & 0.717          & \multicolumn{1}{c|}{0.727}          & 0.093 \\
IDHN~\cite{IDHN}$^\dagger$                    & TMM"20    & 0.827         & 0.823          & 0.822          & \multicolumn{1}{c|}{0.828}          & 0.071          & 0.772          & 0.790          & 0.795          & \multicolumn{1}{c|}{0.803}          & 0.180 \\
Proxy Anchor~\cite{Proxy_anchor_loss}$^\dagger$            & CVPR"20   & 0.796          & 0.831          & 0.834          & \multicolumn{1}{c|}{0.853}          & 0.075          & 0.767          & \uline{0.802}          & 0.809          & \multicolumn{1}{c|}{0.815}          & 0.188 \\
CSQ~\cite{CSQ}$^\dagger$                     & CVPR"20   & 0.795          & 0.819          & 0.849         & \multicolumn{1}{c|}{0.857}          & 0.077          & 0.692          & 0.754          & 0.757          & \multicolumn{1}{c|}{0.769}          & 0.132 \\
OrthoHash~\cite{OrthoHash}$^\dagger$               & NeurIPS"21 & 0.837          & 0.869          & 0.877          & \multicolumn{1}{c|}{\uline{0.891}}          & 0.115          & 0.770          & \uline{0.802}          & \uline{0.810}          & \multicolumn{1}{c|}{\uline{0.825}}          & \uline{0.191} \\
DCILH~\cite{DCILH}$^\dagger$               & TMM"21 & \textbf{0.852}          & \uline{0.879}          & \uline{0.884}          & \multicolumn{1}{c|}{0.888}          & \uline{0.122}          & \uline{0.775}          & 0.793          & 0.797          & \multicolumn{1}{c|}{0.804}          & 0.182 \\\midrule
\textbf{{\name} (Ours)}                    & - & \uline{0.845} & \textbf{0.881} & \textbf{0.893} & \multicolumn{1}{c|}{\textbf{0.901}} & \textbf{0.127} & \textbf{0.794} & \textbf{0.822} & \textbf{0.831} & \multicolumn{1}{c|}{\textbf{0.843}} & \textbf{0.212} \\
\bottomrule
\end{tabular}%
}
\label{Tab.alexnet}
\vspace{-15pt}
\end{table*}

In this section, we describe the datasets used for evaluation, the test protocols, and the implementation details. To evaluate our method, we fairly conduct experiments against existing state-of-the-art~\cite{IDHN, CSQ, OrthoHash, OLAH, Proxy_anchor_loss} and previous methods~\cite{DLBHC, DQN, DHN, HashNet, DMSSPH, NINH, DSH, IAH} on four standard multi-label benchmarks, and justify the superiority of the proposed method both quantitatively and qualitatively. Finally, we explore and conduct in-depth analyses of how each component of the proposed framework contributes to the performance.

\subsection{Implementation Details}

We implement the proposed method in the PyTorch framework~\cite{pytorch} and train on a single NVIDIA RTX 3090 GPU.
We comprehensively adopt AlexNet~\cite{alexnet} and GoogLeNet~\cite{googlenet} pretrained on ImageNet~\cite{ImageNet} as the backbones to justify the robustness of the proposed method. We fine-tune the pretrained backbones for all layers up to the FC layer and map the output layer into $K$-dimensional hash bits. 
We adopt stochastic gradient descent (SGD)~\cite{SGD} to optimize the network with momentum $0.9$ and weight decay $5e-4$. The initial learning rates for optimizing network $\mathcal{F}_{\Theta}$/proxies $\mathscr{P}$ are $0.01$/$0.001$ in AlexNet and $0.02$/$0.02$ in GoogLeNet, respectively. The learning rate decreases by $0.5$ every $10$ epochs with $100$ epochs in total.
\subsection{Dataset \& Evaluation Metrics}

Four standard benchmarks Flickr-25k~\cite{flickr}, VOC-2007~\cite{voc}, VOC-2012~\cite{voc}, and NUS-WIDE~\cite{nuswide} are adopted for evaluation. The statistics of the four datasets are summarized in Tab.~\ref{Tab:dataset}, and the detailed descriptions are as follows.

\myparagraph{Flickr-25k.}
The Flickr-25k dataset contains $25,000$ images. We follow~\cite{IDHN, NINH} to remove the noisy images that do not contain any labels. The remaining $24,851$ images contained $38$ categories in total. Among them, $4,000$ samples are randomly selected as the training set, $1,000$ samples as the query set and the rest images to construct the database.

\myparagraph{VOC-2007 \& VOC-2012.}
VOC-2007 (VOC-2012) contains $9,963$ ($11,540$) images in total, each image attaches to a label containing several of the $20$ categories. We follow~\cite{OLAH} to construct the training set and database of the two datasets for the experiment separately, with $5,011$ ($5,717$) samples in total. The officially provided query set with $4,952$ ($5,823$) samples is used for evaluation.

\myparagraph{NUS-WIDE.}
The NUS-WIDE dataset contains $269,648$ images, and each image is assigned to several $81$ categories. We follow~\cite{IDHN, NINH} to select the most frequent $21$ categories and $195,834$ images containing these attributes. We randomly selected $10,500$ and $2,100$ samples as the training and query set, respectively, and the rest samples are constructed as the database.

\begin{table*}[t!]
\centering
\caption{mAP performance by Hamming Ranking for different hash bits ($\bm{K\in\{16,32,48,64\}}$) in VOC-2007 (mAP@$\bm{5,011}$) and VOC-2012 (mAP@$\bm{5,717}$) with GoogLeNet. $^\ast$: reported results with the same experiment settings from~\cite{OLAH}. $^\dagger$: our reproduced results through the publicly available models. Bold font (underlined) values indicate the best (second best).}
\resizebox{1\linewidth}{!}{%
\scriptsize
\begin{tabular}{c|c|ccccc|ccccc}
\toprule
\multirow{2}{*}{Method} & Dataset  & \multicolumn{5}{c|}{VOC-2007}         & \multicolumn{5}{c}{VOC-2012}          \\ \cmidrule{2-12} 
                        & Pub.     & 16    & 32    & 48    & \multicolumn{1}{c|}{64}    & avg. $\Delta$ & 16    & 32    & 48    & \multicolumn{1}{c|}{64}    & avg. $\Delta$ \\ \midrule
DHN~\cite{DHN}$^\dagger$                     & AAAI''16 & 0.735 & 0.743 & 0.737 & \multicolumn{1}{c|}{0.728} & - & 0.722 & 0.721 & 0.718 & \multicolumn{1}{c|}{0.701} & - \\
NINH~\cite{NINH}$^\ast$                    & CVPR"15 & 0.746 & 0.816 & 0.840 & \multicolumn{1}{c|}{0.851} & 0.077 & 0.731 & 0.788 & 0.809 & \multicolumn{1}{c|}{0.822} & 0.072 \\
DSH~\cite{DSH}$^\ast$                     & CVPR"16 & 0.763 & 0.767 & 0.769 & \multicolumn{1}{c|}{0.775} & 0.033 & 0.753 & 0.766 & 0.776 & \multicolumn{1}{c|}{0.782} & 0.054 \\
IAH~\cite{IAH}$^\ast$                     & TIP"16  & 0.800 & 0.862 & 0.878 & \multicolumn{1}{c|}{0.883} & 0.120 & 0.794 & 0.844 & 0.862 & \multicolumn{1}{c|}{0.864} & 0.126 \\
OLAH~\cite{OLAH}$^\ast$                    & TIP"18  & \uline{0.849} & \uline{0.899} & \uline{0.906} & \multicolumn{1}{c|}{\uline{0.914}} & \uline{0.156} & \uline{0.830} & \uline{0.887} & \uline{0.904} & \multicolumn{1}{c|}{\uline{0.908}} & \uline{0.167} \\
IDHN~\cite{IDHN}$^\dagger$                    & TMM"20  & 0.772 & 0.801 & 0.796 & \multicolumn{1}{c|}{0.772} & 0.050 & 0.785 & 0.805 & 0.797 & \multicolumn{1}{c|}{0.785} & 0.078 \\
Proxy Anchor~\cite{Proxy_anchor_loss}$^\dagger$            & CVPR"20 & 0.752 & 0.802 & 0.836 & \multicolumn{1}{c|}{0.841} & 0.072 & 0.722 & 0.795 & 0.804 & \multicolumn{1}{c|}{0.823} & 0.071 \\
OrthoHash~\cite{OrthoHash}$^\dagger$               & NeurIPS"21 & 0.831          & 0.876          & 0.902          & \multicolumn{1}{c|}{0.909}        & \multicolumn{1}{c|}{0.144}          & 0.823          & 0.885          & 0.893          & \multicolumn{1}{c|}{0.900}          & 0.160 \\

\midrule
\textbf{{\name} (Ours)}                    & -        & \textbf{0.862} & \textbf{0.917} & \textbf{0.932} & \multicolumn{1}{c|}{\textbf{0.937}} & \textbf{0.176} & \textbf{0.841} & \textbf{0.903} & \textbf{0.917} & \multicolumn{1}{c|}{\textbf{0.925}} & \textbf{0.181}  \\ 
\bottomrule
\end{tabular}%
}
\label{Tab.googlenet}
\vspace{-15pt}
\end{table*}

\myparagraph{Evaluation Protocol.}
We follow~\cite{selfdistill, DSRH, NINH, HHF} to employ four metrics for quantitative evaluation: 1). mean average precision (mAP@$N$), 2). precision {\wrt} Top-$N$ returned images (Top-$N$ curves), 3). the average $l_2$ distance of each sample to the corresponding cluster centers ($d_{intra}$), and 4). the average $l_2$ distance of each cluster to the closest irrelevant cluster centers ($d_{inter}$). Regarding mAP@$N$ score computation, we select the Top-$N$ images from the retrieval ranked-list results. The returned images and the query image are considered similar iff they share at least one same label.

\subsection{Quantitative Comparison}
\myparagraph{Baselines \& Settings.}
We compare the proposed method with 1). standard baselines, including HashNet~\cite{HashNet}, DMSSPH~\cite{DMSSPH}, DQN~\cite{DQN}, DLBHC~\cite{DLBHC}, DHN~\cite{DHN}, NINH~\cite{NINH}, DSH~\cite{DSH}, and IAH~\cite{IAH}, 2). state-of-the-art deep hashing methods, including IDHN~\cite{IDHN}, Proxy Anchor~\cite{Proxy_anchor_loss}, CSQ~\cite{CSQ}, OrthoHash~\cite{OrthoHash}, OLAH~\cite{OLAH}, and DCILH~\cite{DCILH}. Note that OLAH and DCILH are two state-of-the-art deep hashing methods specifically designed for multi-label image retrieval. Besides, Proxy Anchor is the state-of-the-art proxy-based method. We verify the robustness of such proxy-based methods in multi-label scenarios to elaborate on how the proposed method improves the metric space and retrieval performance effectively.

Specifically, to justify the effectiveness of the proposed method, we compare methods using AlexNet in Flickr-25k and NUS-WIDE among hash bits $K \in \{12,24,36,48\}$ in Tab.~\ref{Tab.alexnet}, and further compare on another backbone ({\ie}, GoogLeNet) in VOC-2007 and VOC-2012 among $K \in \{16,32,48,64\}$ in Tab.~\ref{Tab.googlenet}, respectively.

\begin{figure}[t!]
    \centering
    \includegraphics[width=1\linewidth]{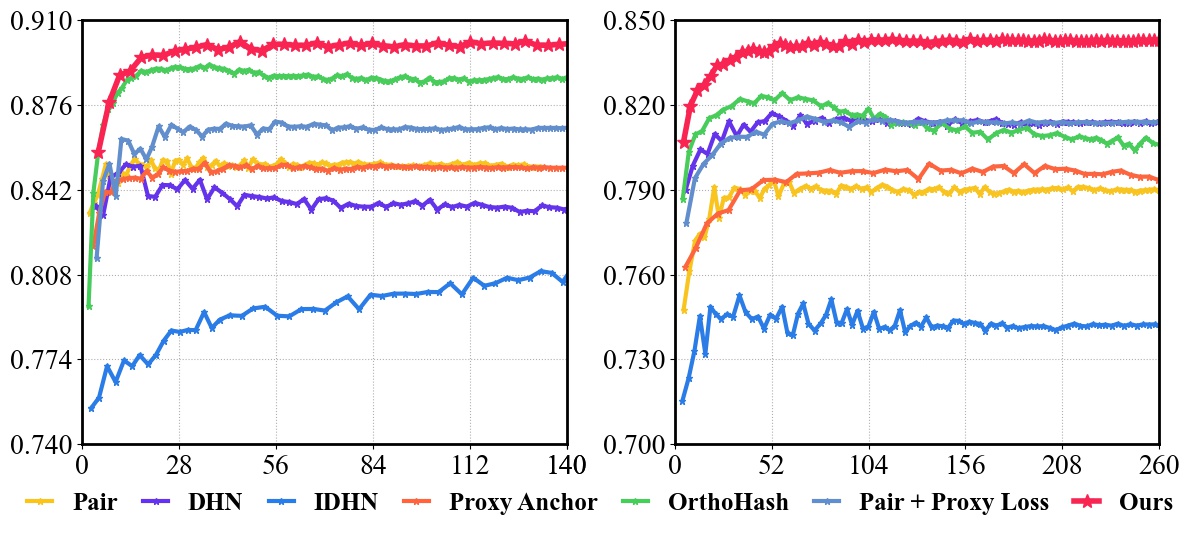}
    \caption{Convergence comparisons in Flickr-25k (left) $\bm{\&}$ NUS-WIDE (right) (48 hash bits with AlexNet). x-axis: training time (sec.), y-axis: mAP@$\Scale[1]{\bm{1,000}}$ (left) $\bm{\&}$ $\Scale[1]{\bm{5,000}}$ (right) performance on the query dataset. {HyP$\Scale[0.9]{\bm{^2}}$ Loss} achieves a faster convergence speed with a more stable training process.}
    \label{fig.converge}
\end{figure}

\myparagraph{Results \& Analysis.}
As Tab.~\ref{Tab.alexnet} and Tab.~\ref{Tab.googlenet} illustrate, {\name} outperforms existing methods over different hash bits in the four benchmarks, which justifies the robustness and effectiveness of the proposed method.
Note that when the hash bits are small ({\eg}, $12$-bit in Flickr-25k/NUS-WIDE and $16$-bit in VOC-2007/VOC-2012), the proposed method achieves $10.20\%$ performance gains on average compared to Proxy Anchor~\cite{Proxy_anchor_loss}, the state-of-the-art proxy-based method in image retrieval.

We justify that the metric space established by proxy loss is insufficient to express the profound label correlations, which achieves unsatisfactory mAP and misclassified retrieval performance. As a comparison, {\name} effectively improves the metric space by additional constraints to explicitly improve the isotropic hypersphere space, and thus improves the retrieval accuracy remarkably.

\subsection{Qualitative Comparison}

\myparagraph{Convergence Comparison.}
To demonstrate that the convergence speed of the proposed method outperforms existing methods, we compare~\cite{DHN, IDHN, Proxy_anchor_loss, OrthoHash} to {\name} in Flickr-25k and NUS-WIDE datasets. The visualized results are presented in Fig.~\ref{fig.converge}.

Fig.~\ref{fig.converge} shows that the proposed method achieves a more stable and faster convergence speed with higher performance compared to the previous state-of-the-art.
Note that pair-based methods~\cite{DHN, IDHN} exhibit training disturbance because the loss function is restricted in a mini-batch and thus lacks generalization, while OrthoHash~\cite{OrthoHash} confronts overfitting risks. As a comparison, Proxy Anchor~\cite{Proxy_anchor_loss} and ours show better stability during the whole training process.

\begin{figure}[t!]
    \centering
    \vspace{-8pt}
    \subfloat[DHN~\cite{DHN}]{
        \includegraphics[width=0.31\linewidth]{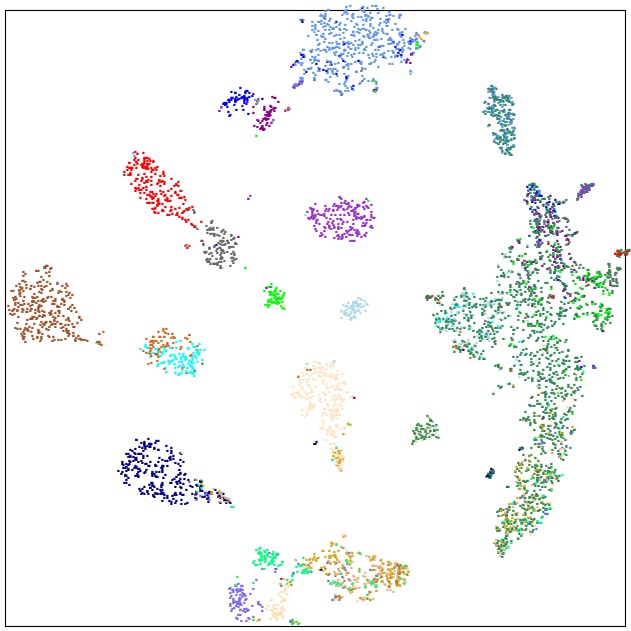}
    }
    \subfloat[IDHN~\cite{IDHN}]{
        \includegraphics[width=0.31\linewidth]{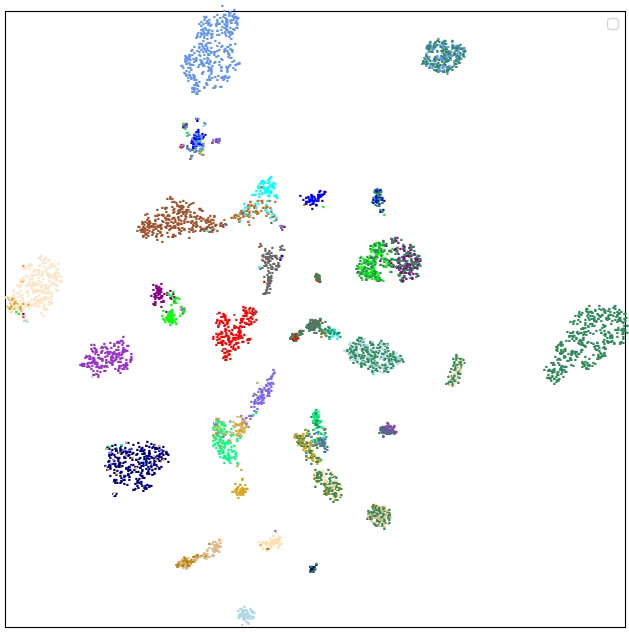}
    }
    \subfloat[Proxy Anchor~\cite{Proxy_anchor_loss}]{
        \includegraphics[width=0.31\linewidth]{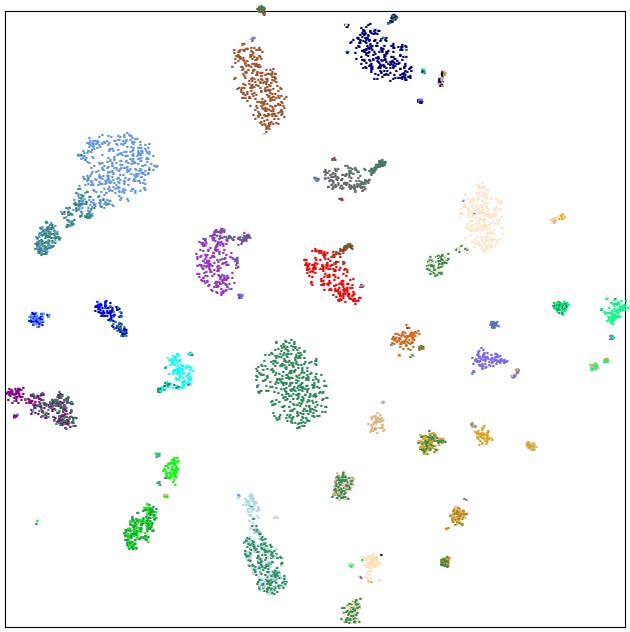}
    }
    \vspace{-12pt}
    \hfill
    \subfloat[Pair Loss]{
        \includegraphics[width=0.31\linewidth]{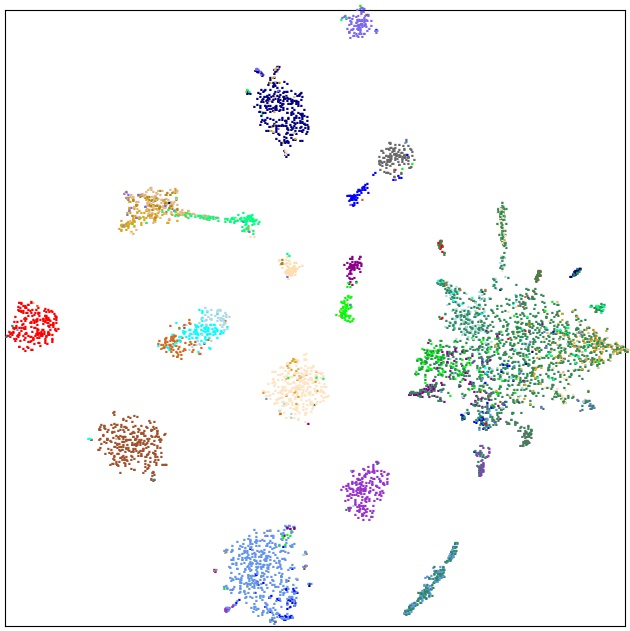}
    }
    \subfloat[Proxy Loss]{
        \begin{overpic}[trim=0cm 0cm 0cm 0cm,clip,width=0.31\linewidth,grid=false]{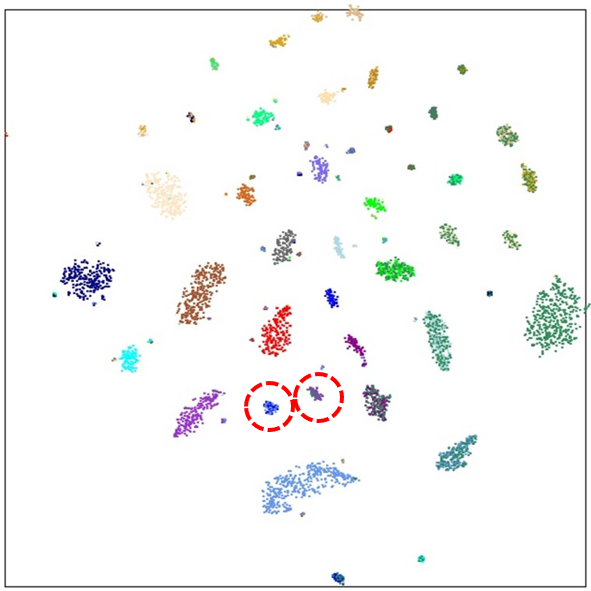}
        \put(35,22){\scriptsize bus}
        \put(30,17){\scriptsize \& car}
        \put(50,38){\scriptsize person}
        \put(68,33){\scriptsize \& train}
        \end{overpic}
        \label{fig.tsnePA}
    }
    \subfloat[{\name} (Ours)]{
        \begin{overpic}[trim=0cm 0cm 0cm 0cm,clip,width=0.31\linewidth,grid=flase]{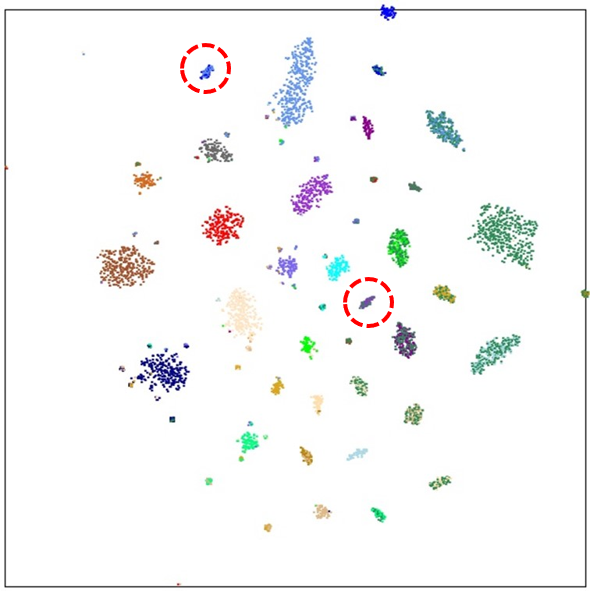}
        \put(17,87){\scriptsize bus}
        \put(16,82){\scriptsize \& car}
        \put(60,55){\scriptsize person}
        \put(66,50){\scriptsize \& train}        
        \end{overpic}
        \label{fig.tsneours}
    }%
    \caption{Visualized t-SNE comparisons in VOC-2007 dataset with $\bm{48}$ hash bits. The scatters of the same color indicate the same categories. The red circles indicate irrelevant samples.}
    \label{fig.tsne}
\end{figure}

\myparagraph{t-SNE Plots.}
To observe how the proposed method contributes better metric space, we use t-SNE~\cite{tsne} to map $K$-dimensional feature vectors into 2D plots. For each sample, we assign different colors around its neighbourhood to present its attributes. Then, the visualized comparisons among DHN~\cite{DHN}, IDHN~\cite{IDHN}, Proxy Anchor~\cite{Proxy_anchor_loss} and pair loss, proxy loss baselines to the proposed {\name} in VOC-2007 dataset are illustrated in Fig.~\ref{fig.tsne}.

Fig.~\ref{fig.tsne} shows that our method achieves visually better data distribution, especially in the confusion samples. The red circles in Fig.~\ref{fig.tsne}(e) and Fig.~\ref{fig.tsne}(f) show how {\name} solves the conflicts in Proxy Loss.
Specifically, the proxy loss improperly embeds irrelevant samples with multiple labels (bus, car) and (person, train) into nearby positions, which damages the retrieval accuracy. As a comparison, {\name} could solve the irrelevant conflicts and thereby achieves visually better metric space with superior performance.

\myparagraph{Top-$\bm{N}$ curve.}
To further demonstrate that {\name} genuinely provides quality search outcomes, we present the precision for the Top-$N$ retrieved images in Fig.~\ref{fig.prcurve}. We can observe that the proposed {\name} consistently establishes the state-of-the-art retrieval performance with higher scores over different hash bits.

\begin{figure*}[t!]
    \includegraphics[width=0.94\textwidth]{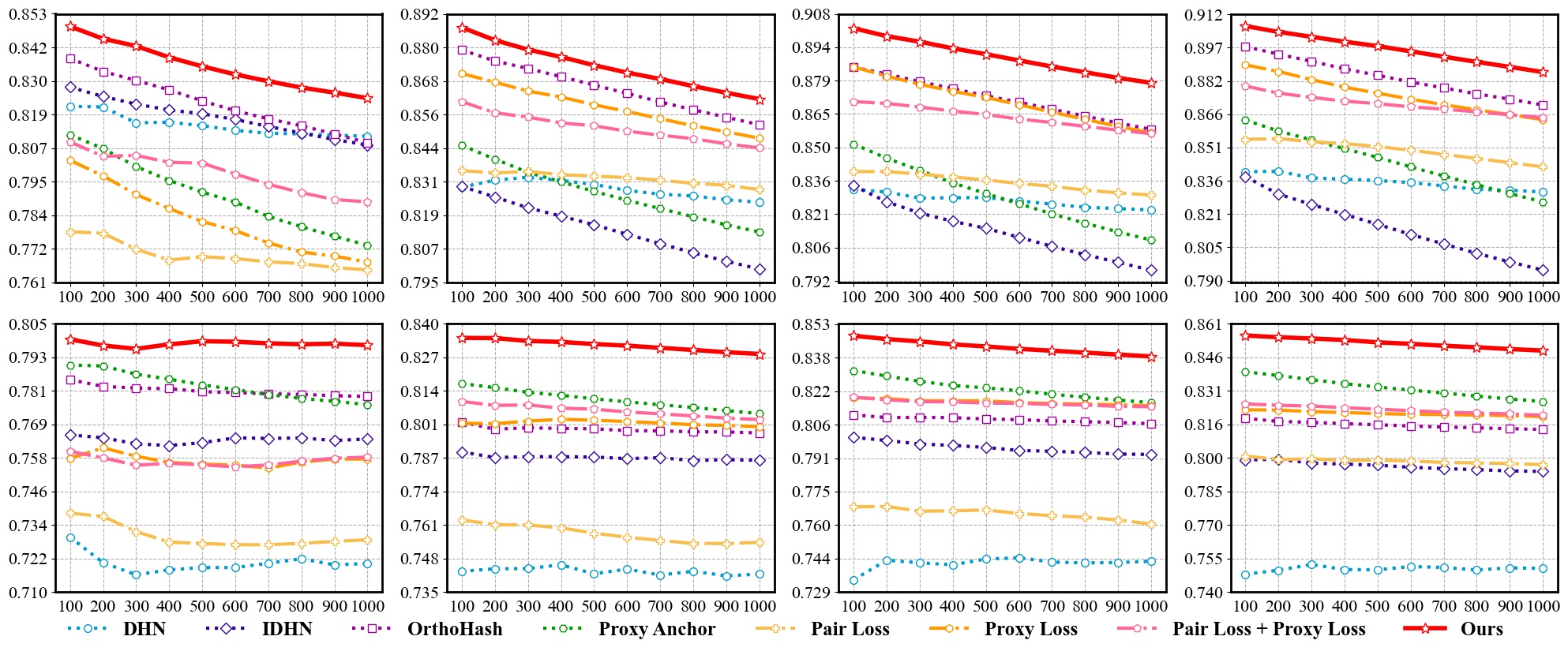}
    \vspace{-3pt}
    \caption{Performance of different methods in Flickr-25k (top) and NUS-WIDE (bottom) datasets. From left to right: Top-$\bm{N}$ curves (x-axis (Top-$\bm{N}$): $\bm{100\rightarrow 1,000}$, y-axis (Precision): $\bm{0\rightarrow 1}$) {\wrt} $\bm{12,24,36,48}$ hash bits, respectively. Our {\name} outperforms previous methods in different datasets among various hash bit lengths consistently.}
    \label{fig.prcurve}
    \vspace{-12pt}
\end{figure*}

\begin{table}[t!]
\setlength{\tabcolsep}{0.6pt}
\centering
\caption{Ablation study $\bm{\&}$ hyperparameter analysis of {HyP$\Scale[0.9]{\bm{^2}}$ Loss}. We report mAP@$\bm{1,000 (5,011)}$ in Flicker-25k (VOC-2007) to show how the proposed method improves the performance. $\bm{\beta=1 (0.5)}$ is empirically best for the two datasets.}
\resizebox{1\linewidth}{!}{%
\scriptsize
\begin{tabular}{l|cccc|cccc}
\toprule
Dataset     & \multicolumn{4}{c|}{Flickr-25k}                                   & \multicolumn{4}{c}{VOC-2007}                                      \\ \midrule
Hash Bits   & 12             & 24             & 36             & 48             & 16             & 32             & 48             & 64             \\ \midrule
Pair Loss   & 0.779          & 0.832          & 0.837          & 0.847          & 0.732          & 0.766          & 0.777          & 0.781          \\
Proxy Loss  & 0.787          & 0.834          & 0.856          & 0.871          & 0.767          & 0.820          & 0.873         & 0.887          \\ 
Proxy + Pair Loss  & 0.817          & 0.851          & 0.857          & 0.870          & 0.818          & 0.843          &  0.875        &  0.881         \\ \midrule
{\name} $(\beta=0.50)$         & 0.838          & 0.876          & \textbf{0.893}          & 0.896          & \textbf{0.862} & \textbf{0.917} & \textbf{0.932} & \textbf{0.937} \\
{\name} $(\beta=0.75)$        & 0.842          & 0.877          & 0.891          & 0.896          & 0.859          & 0.915          & 0.930          & 0.935          \\
{\name} $(\beta=1.00)$           & \textbf{0.845} & \textbf{0.881} & \textbf{0.893} & \textbf{0.901} & 0.857          & 0.912          & 0.926          & 0.934          \\
{\name} $(\beta=1.25)$        & 0.841          & 0.877          & 0.891          & 0.897          & 0.843          & 0.909          & 0.927          & 0.935          \\
%
\bottomrule
\end{tabular}
}
\label{Tab.ablation}
\end{table}
\setlength{\textfloatsep}{0pt}

\subsection{Ablation Study}

To justify how each component of {\name} contributes to a more powerful metric space, we conduct in-depth ablation studies on investigating the effectiveness of \textit{Multi-label Proxy Loss}, \textit{Irrelevant Pair Loss}, and the combination of them with different $\beta$, the results of mAP, $d_{intra}$ and $d_{inter}$ are shown in Tab.~\ref{Tab.ablation} and Tab.~\ref{Tab.ablation2}.

As Tab.~\ref{Tab.ablation} and Tab.~\ref{Tab.ablation2} illustrate, either pair loss or proxy loss fails to achieve satisfactory performance due to its inherent limitations as we analyzed before, and a simple combination of the two terms with $\beta$ is invalid but confronts overwhelmed training overhead. Specifically, since smaller $d_{intra}$ indicates better cluster performance, while larger $d_{inter}$ indicates better disentangle ability on confusion samples. The pair loss constructs a sparse metric space that fails to cluster samples tightly, while the proxy loss fails to distinguish the confusing samples that introduces misclassified results.
As a comparison, the proposed {\name} not only achieves remarkable performance gains, but also ensures better $d_{intra}$ and $d_{inter}$ that establishes superior metric space compared to others, which demonstrates its robustness and retrieval accuracy.
\begin{table}[t!]
\setlength{\tabcolsep}{1.2pt}
\centering
\caption{Ablation study of {\name}. We report $\bm{d_{intra}}$ and $\bm{d_{inter}}$ in Flickr-25k to show {\name} establishes a better metric space that ensures the two metrics simultaneously.}
\resizebox{1\linewidth}{!}{%
\scriptsize
\begin{tabular}{l|cccc|cccc}
\toprule
Metric     & \multicolumn{4}{c|}{$d_{intra}$ $\downarrow$}       & \multicolumn{4}{c}{$d_{inter}$ $\uparrow$}           \\ \midrule
Hash Bits   & 12             & 24             & 36             & 48     & 12             & 24             & 36             & 48      \\ \midrule
Pair Loss  & 3.043 & 3.453 & 4.320 & 4.749 & 2.264 & 2.868 & 3.480 & 4.152 \\
Proxy Loss & 2.094 & 2.695 & 3.174 & 3.763 & 2.564  & 3.322  & 4.144 & 4.890 \\
Proxy + Pair Loss & 2.387    & 2.792 & 3.351  & 3.734  & 2.402 & 2.995  &  	3.524	 & 3.837 \\ \midrule
{\name}       & \textbf{1.859} & \textbf{2.395} & \textbf{2.877} & \textbf{3.237} & \textbf{3.633}  & \textbf{4.331} & \textbf{4.560} & \textbf{5.202} \\ 
\bottomrule
\end{tabular}
\label{Tab.ablation2}
}
\end{table}
\setlength{\textfloatsep}{0pt}

Finally, since the hyperparameter $\beta$ controls the contribution of each component in {\name}, we investigate the influence of $\beta$ with different values and find that $\beta$ should be adjusted in specific datasets but is stable in different hash bits, while {\name} keeps relatively high performance in a wide range $\beta \in [0.50,1.25]$. In Tab.~\ref{Tab.ablation}, our empirical study shows that {\name} performs best in Flickr-25k (VOC-2007) when $\beta=1.00$ ($0.50$), respectively. 

\section{Conclusion}

In this paper, we focus on theoretically analyzing the primary reasons that proxy-based methods are disqualified for multi-label retrieval, and propose the novel {\name} to preserve the efficient training complexity of proxy loss with the irrelevant constraint term, which compensates for the limitation of the hypersphere metric space.
We conduct extensive experiments to justify the superiority of the proposed method in four standard benchmarks with different backbones and hash bits. Both quantitative and qualitative results demonstrate that the proposed {\name} enables fast, reliable and robust convergence speed, and constructs a powerful metric space to improve the retrieval performance significantly.

\noindent\textbf{Acknowledgment.}
This work was supported by SZSTC Grant No. JCYJ20190809172201639 and WDZC20200820200655001, Shenzhen Key Laboratory ZDSYS20210623092001004.

\clearpage

{
\bibliographystyle{ACM-Reference-Format}
\balance
\bibliography{main}
}
\clearpage

\appendix
\section{Appendix}
\subsection{Missing Proof}

\begin{theorem} 
\label{theo:circle}
For the $K$-dimensional metric space $\mathds{R}^K$ with $C$ hypersphere $\mathds{S} \in \mathds{R}^{K}$. The upper bound of distinguishable hypersphere number $\Omega(K,C)$ cannot enumerate the ideal $ \Omega^*(K,C)=\sum\nolimits_{c = 0}^C \tbinom{C}{c} =2^C$ when $C> K+1$. The upper bound is limited at:
\begin{equation}
\tilde{\Omega}(K,C) = \mathop {\sup}\nolimits_{\mathds{S}} {{\Omega}(K,C)}=\tbinom{C-1}{K}+\sum\nolimits_{k=0}^K \tbinom{C}{k} <2^C
\end{equation}
\end{theorem}

\begin{proof}

To begin, we first deduce the optimal distinguishable hyperspace number $\Omega^*(K,C)$ in the $K$-dimensional hyperspace with $C$ categories. Note that each hyperspace represents one category's cluster center. Hence, the ideal distinguishable number of any $c\in \{0,\cdots,C\}$ hyperspheres is the combination of them that is denoted as $n(K,c)$.
\begin{equation}
    n(K,c)= \tbinom{C}{c}
\end{equation}

Hence, according to the binomial theorem, for $C$ hypersphere space, the optimal distinguishable hyperspace number is the summation of each $c\in \{0,\cdots,C\}$, as Eq.~\ref{Eq.max} illustrates.
\begin{equation}
    \Omega^*(K,C)=\sum\nolimits_{c = 0}^C \tbinom{C}{c} = 2^C
    \label{Eq.max}
\end{equation}

Then, note that the ideal distinguishable hyperspace number can only get achieved when the hyperspace is unconstrained ({\ie}, it can have any intersections to each other). When it comes to $K$-dimensional isotropic hypersphere, we raise Lemma.~\ref{lemma} to further demonstrate the upper bound in the $C$ hypersphere metric space.
\begin{lemma}
    When $K\geq 2$, two $K$-dimensional hyperspheres will intersect at one $(K-1)$-dimensional hypersphere at maximum. Specifically, two 2D circles will intersect at one $1$-dimensional hypersphere at maximum. The $1$-dimensional hypersphere is a pair of two points, which are the boundary of a line segment.
    \label{lemma}
\end{lemma}

We denote the $K$-dimensional hypersphere with the number of $C$ has the maximum distinguishable regions as $\Omega(K,C)$. Then, the $K$-dimensional hypersphere with the number of $C-1$ has the maximum distinguishable regions $\Omega(K,C-1)$.
When we add the $C$-th hypersphere into the existing $K$-dimensional hyperspheres with the number of $C-1$, according to Lemma.~\ref{lemma}, the $C$-th hypersphere will intersect with each hypersphere in the hyperspace at one $(K-1)$-dimensional hypersphere at maximum, {\ie}, $(K-1)$-dimensional hyperspheres with the number of $C-1$ at maximum are added into the $K$-dimensional hyperspace. These $(K-1)$-dimensional hyperspheres thereby generate $(K-1)$-dimensional hypersurfaces with the number of $\Omega(K-1,C-1)$ at maximum correspondingly, and each of these new $(K-1)$-dimensional hypersurfaces bisects the $K$-dimensional space into two parts. So we have the following recurrence relation:
\begin{equation}
    \Omega(K,C)=\Omega(K,C-1) + \Omega(K-1,C-1)
    \label{Eq.recursion}
    \end{equation}

Note that one $K$-dimensional hypersphere can only separate hyperspace into two regions, {\ie}, the inside and outside of the hypersphere, respectively. 
Hence, we have the initial condition that for every $K\ge2$ and $C\equiv1$:
\begin{equation}
    \Omega(K, 1)\equiv 2
\label{Eq.K}
\end{equation}

Without loss of generalization, we first consider the simplest particular case when $K=2$. Then, the $K$-dimensional hypersphere is degraded into a 2D circle situation. 

Suppose there are $C-1$ circles in the 2D space, when we have an additional one circle in the space, such that each $C-1$ circle should intersect with the additional one circle to achieve the maximum distinguishable hypersphere number.

Then, according to Lemma.~\ref{lemma}, the circles will at most intersect with the additional circle with $(C-1)$ intersected 1D point pairs. Hence, the $2\times (C-1)$ points will introduce $2\times(C-1)$ lines that segment the original regions into bisects. As a result, the additional one circle will introduce additional $2\times(C-1)$ regions. We can obtain that:
\begin{equation}
\begin{aligned}
    \Omega(2, C)&=\Omega(2, C-1) + 2\times(C-1)\\
                &=\Omega(2, C-2) + 2\times(C-1) + 2\times(C-2)\\
                &=\cdots\\
                &=\Omega(2, 1) + 2\times(C-1+C-2+\cdots+1)
\end{aligned}
\end{equation}

According to Eq.~\ref{Eq.K}, we have $\Omega(2,1)=2$, {\ie}, one 2D circle can represent two distinguishable regions at maximum, then:
\begin{equation}
\begin{aligned}
    \Omega(2, C)&=2 + 2\times(C-1+C-2+\cdots+1)\\
                &=2 + C\times(C-1)\\
                &=C^2 - C + 2\\
                &=\frac{(C-1)(C-2)}{2}+\left(1+C+\frac{C(C-1)}{2}\right)\\
                &=\tbinom{C-1}{2}+\sum\nolimits_{k=0}^2 \tbinom{C}{k}
\label{Eq.2C}
\end{aligned},
\end{equation}
which satisfies the proposition when $K=2$.

Then, we generalize into the $K$-dimensional situation and will prove Theorem.~\ref{theo:circle} by mathematical induction below.

\begin{adjustbox}{minipage=0.97\linewidth,precode=\dbox}
\begin{itemize}[leftmargin=*,nosep,nolistsep]
\item a). Considering the initial situation that for any $K\geq 2, K \in N_{+}$ and $C=1$. As Eq.~\ref{Eq.K} illustrates, we have:
    \begin{equation}
    \Omega(K, C)=\Omega(K, 1)=2=\tbinom{0}{K}+\sum\nolimits_{k=0}^K \tbinom{1}{k},
    \end{equation}
which satisfies the proposition when $C=1$.
\hfill
{$\triangleleft$}%
\item b). For any $K\geq 2, K \in N_{+}$ and a specific $C\geq 1, C\in N_{+}$, we assume the upper bound of distinguishable hypersphere number $\Omega(K,C)$ satisfies:
\begin{equation}
    \Omega(K, C)=\tbinom{C-1}{K}+\sum\nolimits_{k=0}^K \tbinom{C}{k},
\end{equation}
\hfill
{$\triangleleft$}%
\item c). Then, considering the situation of $\Omega(K,C+1)$. When $K=2$, as Eq.~\ref{Eq.2C} illustrates, $\Omega(2,C+1)$ satisfies the proposition. When $K > 2$, according to Eq.~\ref{Eq.recursion}, we have:
\begin{equation}
    \begin{aligned}
    &\Omega(K, C+1) = \Omega(K,C) + \Omega(K-1,C)\\
                 &= \tbinom{C-1}{K}+\sum\nolimits_{k=0}^K \tbinom{C}{k}+ \tbinom{C-1}{K-1}+\sum\nolimits_{k=0}^{K-1}\tbinom{C}{k}\\
                 &= \tbinom{C-1}{K}+\tbinom{C-1}{K-1}+ \tbinom{C}{0} + \sum\nolimits_{k=1}^K \tbinom{C}{k}+ \sum\nolimits_{k=0}^{K-1}\tbinom{C}{k}\\
                 &=\left(\tbinom{C-1}{K}+\tbinom{C-1}{K-1}\right) + \tbinom{C}{0} + \left(\tbinom{C}{0} + \tbinom{C}{1}\right) + \cdots + \left(\tbinom{C}{K-1} + \tbinom{C}{K}\right)
    \end{aligned}
\end{equation}
\end{itemize}
\end{adjustbox}

\begin{adjustbox}{minipage=0.97\linewidth,precode=\dbox}
\begin{itemize}[leftmargin=*,nosep,nolistsep]
\item[] Note that $\tbinom{m}{n}=\tbinom{m-1}{n}+\tbinom{m-1}{n-1}$, then we have:
\begin{equation}
    \begin{aligned}
    \Omega(K, C+1) &= \tbinom{C}{K} + \tbinom{C}{0} + \tbinom{C+1}{1} + \tbinom{C+1}{2} +\cdots+ \tbinom{C+1}{K}\\
                 &= \tbinom{C}{K}+\sum\nolimits_{k=0}^K \tbinom{C+1}{k}
    \end{aligned},
\end{equation}
which satisfies the proposition when $C=C+1$.
\hfill
{$\triangleleft$}%
\end{itemize}
\end{adjustbox}

Finally, according to mathematical induction, for any $K\geq 2, C\geq 1, K,C \in N_{+}$, the upper bound of distinguishable hypersphere number in $K$-dimensional metric space is obtained that:
\begin{equation}
\begin{aligned}
    \tilde{\Omega}(K,C) = \mathop {\sup}\nolimits_{\mathds{S}} {{\Omega}(K,C)}=\tbinom{C-1}{K}+\sum\nolimits_{k=0}^K \tbinom{C}{k}
\end{aligned}
\end{equation}

When C = K + 1, we can see that the ideal distinguishable hyperspace number ${\Omega}^*(K,C)$ is equal to the upper bound $\tilde{\Omega}(K,C)$ because the hash bit length is large enough to enumerate all possible situations, as Eq.~\ref{Eq.K+1} illustrates.
\begin{equation}
    \begin{aligned}
        \tilde{\Omega}(K,C)&=\tbinom{K+1-1}{K}+\sum\nolimits_{k=0}^K \tbinom{K+1}{k}\\
                   &=\tbinom{K}{K} + \tbinom{K+1}{0} + \cdots + \tbinom{K+1}{K}\\
                   &=\tbinom{K+1}{0} + \cdots + \tbinom{K+1}{K}+\tbinom{K+1}{K+1} \\
                   &=\sum\nolimits_{k=0}^{K+1} \tbinom{K+1}{k}=2^{K+1}\\
                   &=2^C ={\Omega}^*(K,C)
    \label{Eq.K+1}
    \end{aligned}
\end{equation}

Then, we will prove that when $C\textgreater K+1$, $\tilde{\Omega}(K,C)\textless {\Omega}^*(K,C)$ by mathematical induction below.

\begin{adjustbox}{minipage=0.97\linewidth,precode=\dbox}
\begin{itemize}[leftmargin=*,nosep,nolistsep]
\item a). Considering the initial situation that for any $K\geq 2, K \in N_{+}$, when $C = K + 2$, according to Eq.~\ref{Eq.recursion}, we have:
\begin{equation}
    \begin{aligned}
        &\Omega(K,C)=\Omega(K,K+2)\\
                   &=\Omega(K,K+1) + \Omega(K-1,K+1)\\
                   &=\Omega(K,K+1) + \Omega(K-1,K) + \Omega(K-2,K)\\
                   &=\cdots \\
                   &=\Omega(K,K+1) + \Omega(K-1,K) +\cdots +\Omega(3,4) + \Omega(2, 4)
    \end{aligned}
\end{equation}
As Eq.~\ref{Eq.2C} and Eq.~\ref{Eq.K+1} illustrate, we have:
\begin{equation}
    \begin{aligned}
        \Omega(K,C)&=2^{K+1} + 2^K + \cdots + 2^4 + 4^2 -4+2\\
                   &=2^{K+2} -2\\
                   &=2^{C} -2 \\
                   &< 2^{C} = {\Omega}^*(K,C),
    \end{aligned}
\end{equation}
which satisfies the proposition when $C=K+2$.
\hfill
{$\triangleleft$}%
\item b). For any $K\geq 2, K \in N_{+}$ and a specific $I\geq 2, I\in N_{+}$, let $C=K+I$, we assume the inequation $\tilde{\Omega}(K,C) < \Omega^*(K,C)$ satisfies:
\begin{equation}
    \tilde{\Omega}(K,C)=\tilde{\Omega}(K,K+I) < \Omega^*(K,K+I) = 2^{K+I},
\end{equation}
\hfill
{$\triangleleft$}%
\item c). Then, when $C=K+I+1$, we have:
\begin{equation}
    \begin{aligned}
        &\Omega(K,C)=\Omega(K,K+I+1)\\
                   &=\Omega(K,K+I) + \Omega(K-1,K+I)\\
                   &=\Omega(K,K+I) + \Omega(K-1,K+I-1) + \Omega(K-2,K+I-1)\\
                   &=\cdots
    \end{aligned}
\end{equation}
\end{itemize}
\end{adjustbox}

\begin{adjustbox}{minipage=0.97\linewidth,precode=\dbox}
\begin{itemize}[leftmargin=*,nosep,nolistsep]
\item[]
\begin{equation}
    \begin{aligned}
    \nonumber
                   &=\Omega(K,K+I) + \Omega(K-1,K+I-1) +\cdots +\Omega(3,3+I)\\
                   &\quad  + \Omega(2, 3+I)\\
                   &<2^{K+I} + 2^{K+I-1} + \cdots + 2^{3+I} + (3+I)^2-(3+I)+2\\
                   &=2^{K+I+1} - 2^{3+I} + I^2+5I+8
    \end{aligned}
\end{equation}
Note that $2^{3+I}\textgreater I^2+5I+8$ when $I\geq2$, then we have:
\begin{equation}
    \begin{aligned}
        \Omega(K,C)&<2^{K+I+1}=\Omega^*(K,C),
    \end{aligned}
\end{equation}
which satisfies the proposition when $C=K+I+1$.
\hfill
{$\triangleleft$}%
\end{itemize}
\end{adjustbox}

Finally, according to mathematical induction, $\tilde{\Omega}(K,C)< {\Omega}^*(K,C)$ when $C\textgreater K+1$.

\end{proof}

\begin{theorem}
\label{theo:middle}
When the proxies have converged to fixed positions (i.e., the angles of proxy pairs are constant), then the best position of $2$-label samples is the middle of $2$ positive proxies. The $n$-label scenarios can be deduced in a similar fashion.
\end{theorem}

\begin{proof}
Suppose the feature vector $\bm{v_x}$ contains $2$ labels $y_1$, $y_2$, with corresponding proxies $\bm{p}_1$, $\bm{p}_2$. 
Let $\theta_1 = \langle\bm{v_x}, \bm{p}_1\rangle$, $\theta_2 = \langle\bm{v_x}, \bm{p}_2\rangle$, $\theta_3 = \langle\bm{p}_1, \bm{p}_3\rangle$,  where $\theta_1, \theta_2, \theta_3 \in [0, \pi]$. 

Note that the effect of negative proxies is negligible when about convergence, because they are away from $\bm{v_x}$, so we only consider the gradient from $\bm{p}_1$, $\bm{p}_2$.

Then we have $\mathcal{L}_+ = -(\cos\theta_1 + \cos\theta_2)$, and the objective function is $\arg \max(\cos\theta_1 + \cos\theta_2)$ accordingly. 

If $\bm{v_x}$ is non-coplanar with $\bm{p}_1$, $\bm{p}_2$, we have the projection $\bm{v_x}'$ in the plane defined by $\bm{p}_1, \bm{p}_2$, and denote the corresponding angles with $\bm{p}_1$, $\bm{p}_2$ as $\theta_1'$, $\theta_2'$, such that $\theta_1' \textless \theta_1, \theta_2' \textless \theta_2 \Rightarrow \cos\theta_1' + \cos\theta_2' \textgreater \cos\theta_1 + \cos\theta_2$.

Hence, the optimal objective function is satisfied when $\bm{v_x}, \bm{p}_1, \bm{p}_2$ are coplanar such that $\theta_1' = \theta_1, \theta_2' = \theta_2$, and ensures $\theta_3 = \theta_1 + \theta_2$. Then we have:
\begin{equation}
    \arg\min \mathcal{L}_+ = \arg\max((1 + \cos\theta_3)\cos\theta_1+\sin\theta_3\sqrt{1 - \cos^2\theta_1})
\end{equation}
To obtain the extreme point of $\mathcal{L}_+$, let:
\begin{equation}
    \frac{\partial\mathcal{L}_+}{\partial \cos\theta_1} = 1 + \cos\theta_3 - \sin\theta_3\frac{\cos\theta_1}{\sqrt{1-\cos^2\theta_1}} = 0
\end{equation}
Note that $1 + \cos\theta_3\geq 0$, $\sin\theta_3\geq 0$, $\sqrt{1 - \cos^2\theta_1} \geq 0$, we can get $\cos\theta_1 \geq 0$. Thus $\theta_1 \in [0, \frac{\pi}{2}]$. Then, it is easy to obtain that: 
\begin{equation}
\begin{aligned}
     &\frac{1-\cos^2\theta_3}{2+2\cos\theta_3} = 1 - \cos^2\theta_1 \\
     &\Rightarrow \cos^2\theta_1 = 1 - \frac{1 - \cos^2\theta_3}{2+2\cos\theta_3}\\
     &= \cos^2\frac{\theta_3}{2}    
\end{aligned}
\end{equation}

Considering the domain of $\theta_1, \theta_2, \theta_3$, we have $\cos\theta_1 = \cos\frac{\theta_3}{2} \Rightarrow \theta_1 = \theta_2 = \frac{\theta_3}{2}$, {\ie}, $\bm{v_x}$ will be embedded into the middle of the two proxies.

Similarly, we can extend the conclusion into $n$-label scenarios where the optimal solution is satisfied when $\bm{v_x}^*$ is in the middle of $n$-proxies, as we claimed in the main paper.

\end{proof}

\end{document}